\documentclass[11pt]{article}

\usepackage[colorlinks=true,linkcolor=blue,urlcolor=black,citecolor=blue,breaklinks]{hyperref}
\usepackage{fullpage}
\usepackage[T1]{fontenc}
\usepackage{color}
\usepackage{enumerate}
\usepackage{graphicx}
\usepackage{algorithm,algpseudocode}
\usepackage{varwidth}
\usepackage{mathrsfs}
\usepackage{mathtools}
\usepackage[labelfont=bf,font={small}]{caption}
\usepackage[labelfont=bf]{subcaption}

\usepackage{amsmath,amsthm,amssymb,colonequals,etoolbox}

\usepackage[square,numbers]{natbib}
\newtoggle{draft}
\togglefalse{draft}

\newtheorem{thm}{Theorem}[section]

\newtheorem{lem}[thm]{Lemma}
\newtheorem{prop}[thm]{Proposition}

\newtheorem{defn}{Definition}
\newtheorem{cor}[thm]{Corollary}

\newtheorem{assume}{Assumption}

\numberwithin{equation}{section}

\newcommand*\samethanks[1][\value{footnote}]{\footnotemark[#1]}

\DeclareMathOperator*{\Tr}{\mathrm{tr}}

\DeclareMathOperator*{\polylog}{polylog}

\newcommand{\R}{\ensuremath{\mathbb{R}}}

\newcommand{\norm}[1]{\lVert #1 \rVert}
\newcommand{\bignorm}[1]{\left\lVert #1 \right\rVert}

\newcommand{\ip}[2]{\ensuremath{\langle #1, #2 \rangle}}
\newcommand{\bigip}[2]{\left\langle #1, #2 \right\rangle}

\newcommand{\Cov}{\mathrm{Cov}}
\newcommand{\E}{\mathbb{E}}
\newcommand{\abs}[1]{\ensuremath{| #1 |}}
\newcommand{\bigabs}[1]{\ensuremath{\left| #1 \right|}}

\newcommand{\ceil}[1]{\lceil #1 \rceil}
\newcommand{\bigceil}[1]{\left\lceil #1 \right\rceil}

\newcommand{\ind}{\mathbf{1}}

\newcommand{\svec}{\mathrm{svec}}

\newcommand{\smat}{\mathrm{smat}}

\renewcommand{\Pr}{\mathbb{P}}
\newcommand{\T}{\mathsf{T}}

\newcommand{\smin}{\sigma_{\mathrm{min}}}

\newcommand{\calN}{\mathcal{N}}
\newcommand{\calE}{\mathcal{E}}

\newcommand{\calF}{\mathcal{F}}

\newcommand{\calS}{\mathcal{S}}

\newcommand{\calD}{\mathcal{D}}

\newcommand{\calO}{\mathcal{O}}

\newcommand{\cvectwo}[2]{\begin{bmatrix} #1 \\ #2 \end{bmatrix}}
\newcommand{\rvectwo}[2]{\begin{bmatrix} #1 & #2 \end{bmatrix}}
\newcommand{\bmattwo}[4]{\begin{bmatrix} #1 & #2 \\ #3 & #4 \end{bmatrix}}

\newcommand{\Ah}{\widehat{A}}
\newcommand{\Bh}{\widehat{B}}

\newcommand{\Ph}{\widehat{P}}
\newcommand{\Kh}{\widehat{K}}

\newcommand{\dlyap}{\mathsf{dlyap}}

\newcommand{\Avg}{\mathsf{Avg}}

\newcommand{\Otilde}{\widetilde{O}}

\newcommand{\Keval}{K_{\mathrm{eval}}}
\newcommand{\Kplay}{K_{\mathrm{play}}}
\newcommand{\Lplay}{L_{\mathrm{play}}}

\newcommand{\qh}{\widehat{q}}
\newcommand{\Qh}{\widehat{Q}}
\newcommand{\Qhat}{\widehat{Q}}

\newcommand{\Tmult}{T_{\mathrm{mult}}}

\newcommand{\Kbar}{\overline{K}}
\newcommand{\Vbar}{\overline{V}}
\newcommand{\rhobar}{\overline{\rho}}
\newcommand{\sigmabar}{\overline{\sigma}}
\newcommand{\varepsilonbar}{\overline{\varepsilon}}

\newcommand{\Creq}{C_{\mathrm{req}}}
\newcommand{\Cerr}{C_{\mathrm{err}}}

\newcommand{\Kmax}{K_{\mathrm{max}}}
\newcommand{\Vmax}{V_{\mathrm{max}}}

\newcommand{\vertiii}[1]{{\left\vert\kern-0.25ex\left\vert\kern-0.25ex\left\vert #1 
    \right\vert\kern-0.25ex\right\vert\kern-0.25ex\right\vert}}

\begin{document}

\title{Finite-time Analysis of Approximate Policy Iteration for the Linear Quadratic Regulator}

\author{Karl Krauth\thanks{Both authors contributed equally to this work.}~, Stephen Tu\samethanks[1]~, and Benjamin Recht \\
University of California, Berkeley}
\maketitle


\begin{abstract}
We study the sample complexity of approximate policy iteration (PI) for the 
Linear Quadratic Regulator (LQR), building on a recent line of work using
LQR as a testbed to understand the limits of reinforcement learning (RL)
algorithms on continuous control tasks.
Our analysis quantifies the tension between policy improvement and
policy evaluation, and suggests that policy evaluation is the
dominant factor in terms of sample complexity.
Specifically, we show that to obtain a controller that is within
$\varepsilon$ of the optimal LQR controller,
each step of policy evaluation requires at most $(n+d)^3/\varepsilon^2$ samples,
where $n$ is the dimension of the state vector and $d$ is the dimension of the input vector. 
On the other hand, only $\log(1/\varepsilon)$ policy improvement steps suffice,
resulting in an overall sample complexity of $(n+d)^3 \varepsilon^{-2} \log(1/\varepsilon)$.
We furthermore build on our analysis and construct a simple adaptive procedure based on $\varepsilon$-greedy exploration which relies
on approximate PI
as a sub-routine and obtains $T^{2/3}$ regret, improving upon a recent
result of \citet{abbasi18}.
\end{abstract}

\section{Introduction}

With the recent successes of reinforcement learning (RL) on continuous
control tasks, there has been a renewed interest in understanding the
sample complexity of RL methods.
A recent line of work
has focused on the Linear Quadratic Regulator (LQR)
as a testbed to understand the behavior and trade-offs of various RL algorithms in the continuous state and action space setting.
These results can be broadly grouped into two categories: 
(1) the study of \emph{model-based} methods which use data to build an estimate
of the transition dynamics, and (2) \emph{model-free} methods
which directly estimate the optimal feedback controller from data without building a dynamics model as an intermediate step.
Much of the recent progress in LQR has focused on the model-based side,
with an analysis of robust control from \citet{dean17}
and certainty equivalence control by \citet{fiechter97} and \citet{mania19}.
These techniques have also been extended to the online, adaptive setting
\cite{abbasi11,dean18,cohen19,abeille18,ouyang17}.
On the other hand, for classic model-free RL algorithms
such as Q-learning, SARSA, and approximate policy iteration (PI),
our understanding is much less complete despite the fact that these algorithms
are well understood in the tabular (finite state and action space) setting.
Indeed, most of the model-free analysis for LQR \cite{fazel18,malik19,tu18b} 
has focused exclusively on derivative-free random search methods.

In this paper, we extend our understanding of model-free algorithms for LQR by studying the 
performance of approximate PI on LQR, which is a 
classic approximate dynamic programming algorithm. Approximate PI
is a model-free algorithm which iteratively uses trajectory data to estimate
the state-value function associated to the current policy (via e.g.~temporal difference learning), and then
uses this estimate to greedily improve the policy. A key issue in analyzing approximate PI
is to understand the trade-off between the number of policy improvement iterations, and 
the amount of data to collect for each policy evaluation phase.
Our analysis quantifies this trade-off, showing that if 
least-squares temporal difference learning (LSTD-Q) \cite{boyan99,lagoudakis03} is used for policy evaluation, then
a trajectory of length $\Otilde((n+d)^3/\varepsilon^2)$ for each inner step
of policy evaluation combined with $\calO(\log(1/\varepsilon))$ outer steps of policy improvement 
suffices to learn a controller that has $\varepsilon$-error from the optimal controller.
This yields an overall sample complexity of $\calO( (n+d)^3 \varepsilon^{-2}\log(1/\varepsilon))$.
Prior to our work, the only known guarantee for approximate PI on LQR was the
asymptotic consistency result of \citet{bradtke94} in the setting of no process noise.

We also extend our analysis of approximate PI to the online, adaptive LQR
setting popularized by \citet{abbasi11}. By using a greedy exploration scheme
similar to \citet{dean18} and \citet{mania19}, we prove a $\Otilde(T^{2/3})$ regret bound for
a simple adaptive policy improvement algorithm. While the $T^{2/3}$ rate
is sub-optimal compared to the $T^{1/2}$ regret from model-based methods
\cite{abbasi11,cohen19,mania19}, our analysis improves the $\Otilde(T^{2/3+\varepsilon})$
regret (for $T \geq C^{1/\varepsilon}$) from the model-free Follow the Leader (FTL) algorithm of \citet{abbasi18}. To the best of our knowledge, we give the best regret guarantee known for
a model-free algorithm. We leave open the question of whether or not a model-free
algorithm can achieve optimal $T^{1/2}$ regret.


\section{Main Results}

In this paper, we consider the following linear dynamical system:
\begin{align}
	x_{t+1} = A x_t + B u_t + w_t \:, \:\: w_t \sim \calN(0, \sigma_w^2 I) \:, \:\: x_0 \sim \calN(0, \Sigma_0) \:. \label{eq:dynamics}
\end{align}
We let $n$ denote the dimension of the state $x_t$
and $d$ denote the dimension of the input $u_t$. For simplicity we assume that $d \leq n$,
e.g. the system is under-actuated.
We fix two positive definite cost matrices $(S, R)$, and 
consider the infinite horizon average-cost Linear Quadratic Regulator (LQR):
\begin{align}
	J_\star := \min_{\{u_t(\cdot)\}} \lim_{T \to \infty} \E\left[ \frac{1}{T} \sum_{t=1}^{T} x_t^\T S x_t + u_t^\T R u_t \right] \:\: \text{subject to} \:\: \eqref{eq:dynamics} \:. \label{eq:lqr_problem}
\end{align}
We assume the dynamics matrices $(A,B)$ are unknown to us, and our method of interaction
with \eqref{eq:dynamics} is to choose an input sequence $\{u_t\}$ and observe the resulting
states $\{x_t\}$.

We study the solution to \eqref{eq:lqr_problem} using \emph{least-squares policy iteration (LSPI)},
a well-known approximate dynamic programming method in RL introduced by \citet{lagoudakis03}. 
The study
of approximate PI on LQR dates back to the Ph.D. thesis of Bradtke~\cite{bradtke94},
where he showed that for \emph{noiseless} LQR (when $w_t = 0$ for all $t$),
the approximate PI algorithm is asymptotically consistent.
In this paper we expand on this result and quantify non-asymptotic rates for approximate PI on LQR.

\paragraph{Notation.}
For a positive scalar $x > 0$, we let $x_+ = \max\{1, x\}$.
A square matrix $L$ is called stable if $\rho(L) < 1$ where
$\rho(\cdot)$ denotes the spectral radius of $L$. For a symmetric matrix $M \in \R^{n \times n}$, we let
$\dlyap(L, M)$ denote the unique solution to the
discrete Lyapunov equation $P = L^\T P L + M$.
We also let $\svec(M) \in \R^{n(n+1)/2}$ denote
the vectorized version of the upper triangular part of $M$
so that $\norm{M}_F^2 = \ip{\svec(M)}{\svec(M)}$.
Finally, $\smat(\cdot)$ denotes the inverse of $\svec(\cdot)$, so that
$\smat(\svec(M)) = M$.

\subsection{Least-Squares Temporal Difference Learning (LSTD-Q)}
\label{sec:results:lstdq}

The first component towards an understanding of approximate PI is to understand least-squares temporal difference learning (LSTD-Q) for $Q$-functions, which is the fundamental building block of LSPI. Given a policy $\Keval$ which stabilizes $(A,B)$, the goal of LSTD-Q is to estimate the parameters of the $Q$-function associated to $\Keval$. Bellman's equation for infinite-horizon average cost MDPs (c.f.~\citet{bertsekas07}) states that the (relative) $Q$-function associated to a policy $\pi$ satisfies the following fixed-point equation:
\begin{align}
	\lambda + Q(x, u) = c(x, u) + \E_{x' \sim p(\cdot|x,u)}[ Q(x', \pi(x')) ] \:. \label{eq:bellman_q}
\end{align}
Here, $\lambda \in \R$ is a free parameter chosen so that the fixed-point equation holds.
LSTD-Q operates under the \emph{linear architecture} assumption, which states that
the $Q$-function can be described as $Q(x, u) = q^\T \phi(x, u)$, for a known (possibly non-linear)
feature map $\phi(x, u)$. It is well known that LQR satisfies the linear architecture assumption, since we have:
\begin{align*}
	Q(x, u) &= \svec(Q)^\T \svec\left( \cvectwo{x}{u}\cvectwo{x}{u}^\T \right)  \:, \:\: Q = \bmattwo{S}{0}{0}{R} + \cvectwo{A^\T}{B^\T} V \rvectwo{A}{B} \:, \\
	V &= \dlyap(A + B\Keval, S + \Keval^\T R \Keval) \:, \:\: \lambda = \bigip{Q}{\sigma_w^2 \cvectwo{I}{\Keval}\cvectwo{I}{\Keval}^\T} \:.
\end{align*}
Here, we slightly abuse notation and let $Q$ denote the $Q$-function and also the matrix parameterizing the $Q$-function. Now suppose that a trajectory $\{(x_t, u_t, x_{t+1})\}_{t=1}^{T}$ is collected.
Note that LSTD-Q is an \emph{off-policy} method (unlike the closely related LSTD estimator for value functions), and therefore the inputs $u_t$ can come from any sequence that provides sufficient excitation
for learning. In particular, it does \emph{not} have to come from the policy $\Keval$. In this
paper, we will consider inputs of the form:
\begin{align}
	u_t = \Kplay x_t + \eta_t \:, \:\: \eta_t \sim \calN(0, \sigma_\eta^2 I) \:,
\end{align}
where $\Kplay$ is a stabilizing controller for $(A,B)$. Once again we emphasize that
$\Kplay \neq \Keval$ in general. The injected noise $\eta_t$ is needed in order to provide
sufficient excitation for learning.
In order to describe the LSTD-Q estimator, we define the following quantities
which play a key role throughout the paper:
\begin{align*}
	\phi_t &:= \phi(x_t, u_t) \:, \:\: \psi_t := \phi(x_t, \Keval x_t) \:, \\
	 f &:= \svec\left(\sigma_w^2 \cvectwo{I}{\Keval}\cvectwo{I}{\Keval}^\T\right) \:, \:\: c_t := x_t^\T S x_t + u_t^\T R u_t \:.
\end{align*}
The LSTD-Q estimator estimates $q$ via:
\begin{align}
	\qh := \left( \sum_{t=1}^{T} \phi_t (\phi_t - \psi_{t+1} + f)^\T \right)^{\dag} \sum_{t=1}^{T} \phi_t c_t \:. \label{eq:lstdq_estimator}
\end{align}
Here, $(\cdot)^{\dag}$ denotes the Moore-Penrose pseudo-inverse. Our first result
establishes a non-asymptotic bound on the quality of the estimator $\qh$, measured
in terms of $\norm{\qh-q}$. Before we state our result, we introduce a key definition
that we will use extensively.
\begin{defn}
\label{def:tau_rho_stable}
Let $L$ be a square matrix.
Let $\tau \geq 1$ and $\rho \in (0, 1)$. We say that $L$ is $(\tau, \rho)$-stable
if
\begin{align*}
	\norm{L^k} \leq \tau \rho^k \:, \:\: k=0, 1, 2, ... \:.
\end{align*}
\end{defn}
While stability of a matrix is an asymptotic notion,
Definition~\ref{def:tau_rho_stable} quantifies the degree of stability by
characterizing the transient response of the powers of a matrix by the parameter $\tau$.
It is closely related to the notion of \emph{strong stability} from \citet{cohen19}.

With Definition~\ref{def:tau_rho_stable} in place, we state our first result
for LSTD-Q.
\begin{thm}
\label{thm:lstd_q_estimation}
Fix a $\delta \in (0, 1)$. Let policies $\Kplay$ and $\Keval$ stabilize $(A,B)$,
and assume that both $A+B\Kplay$ and $A+B\Keval$ are $(\tau,\rho)$-stable. Let the initial state $x_0 \sim \calN(0, \Sigma_0)$
and consider the inputs $u_t = \Kplay x_t + \eta_t$ with $\eta_t \sim \calN(0, \sigma_\eta^2 I)$.
For simplicity, assume that $\sigma_\eta \leq \sigma_w$. Let $P_\infty$ denote
the steady-state covariance of the trajectory $\{x_t\}$:
\begin{align}
	P_\infty = \dlyap((A+B\Kplay)^\T, \sigma_w^2 I + \sigma_\eta^2 BB^\T) \:.
\end{align}
Define the proxy variance $\sigmabar^2$ by:
\begin{align}
	\sigmabar^2 := \tau^2\rho^4\norm{\Sigma_0} + \norm{P_\infty} + \sigma_\eta^2 \norm{B}^2 \:. \label{eq:proxy_variance}
\end{align}
Suppose that $T$ satisfies:
\begin{align}
    T \geq \Otilde(1) \max\bigg\{ (n+d)^2, \frac{\tau^4}{\rho^4(1-\rho^2)^2} \frac{(n+d)^4}{\sigma_\eta^4} \sigma_w^2 \sigmabar^2 \norm{\Kplay}_+^4 \norm{\Keval}_+^8  (\norm{A}^4+\norm{B}^4)_+
    \bigg\} \:. \label{eq:lstdq_T_requirement}
\end{align}
Then we have with probability at least $1-\delta$,
\begin{align}
    \norm{\qh - q} \leq \Otilde(1) \frac{\tau^2}{\rho^2(1-\rho^2)} \frac{(n+d)}{\sigma_\eta^2\sqrt{T}}\sigma_w \sigmabar \norm{\Kplay}_+^2 \norm{\Keval}_+^4 (\norm{A}^2+\norm{B}^2)_+\norm{Q^{\Keval}}_F \:. \label{eq:lstdq_T_bound}
\end{align}
Here the $\Otilde(1)$ hides $\polylog( n, \tau, \norm{\Sigma_0}, \norm{P_\infty}, \norm{\Kplay}, T/\delta, 1/\sigma_\eta )$ factors.
\end{thm}
Theorem~\ref{thm:lstd_q_estimation} states that:
\begin{align*}
T \leq \Otilde\left( (n+d)^4, \frac{1}{\sigma_\eta^4} \frac{(n+d)^3}{\varepsilon^2} \right)
\end{align*}
timesteps are sufficient to achieve error $\norm{\qh - q} \leq \varepsilon$ w.h.p.
Several remarks are in order. First, while the $(n+d)^4$ burn-in is likely sub-optimal,
the $(n+d)^3/\varepsilon^2$ dependence is sharp as shown by
the asymptotic results of \citet{tu18b}. Second, the $1/\sigma_\eta^4$ dependence on the injected excitation noise will be important when we study the online, adaptive setting in Section~\ref{sec:results:adaptive}.
We leave improving the polynomial dependence of the burn-in period to future work.

The proof of Theorem~\ref{thm:lstd_q_estimation} appears in
Section~\ref{sec:lstdq} and rests on top of several recent advances.
First, we build off the work of \citet{abbasi18} to derive a new basic inequality for LSTD-Q
which serves as a starting point for the analysis. Next, we combine the small-ball techniques of
\citet{simchowitz18} with the self-normalized martingale inequalities of \citet{abbasi11b}.
While an analysis of LSTD-Q is presented in \citet{abbasi18} (which builds on the analysis for LSTD from \citet{tu18a}), a direct application of their
result yields a $1/\sigma_\eta^8$ dependence; the use of self-normalized inequalities is necessary
in order to reduce this dependence to $1/\sigma_\eta^4$.

\subsection{Least-Squares Policy Iteration (LSPI)}
\label{sec:results:lspi_offline}

With Theorem~\ref{thm:lstd_q_estimation} in place, we are ready to present the main results
for LSPI. 
We describe two versions of LSPI in Algorithm~\ref{alg:lspi_offline_v1} and
Algorithm~\ref{alg:lspi_offline_v2}.

\begin{minipage}[t]{0.46\textwidth}
  \begin{algorithm}[H]
    \caption{$\mathsf{LSPIv1}$ for LQR}
    \begin{algorithmic}[1]
        \Require $K_0$: initial stabilizing controller,
        	\Statex ~~~~~~$N$: number of policy iterations,
        	\Statex ~~~~~~$T$: length of rollout,
        	\Statex ~~~~~~$\sigma_\eta^2$: exploration variance,
        	\Statex ~~~~~~$\mu$: lower eigenvalue bound.
        \State{Collect $\calD = \{ (x_k, u_k, x_{k+1}) \}_{k=1}^{T}$ with input $u_k = K_0 x_k + \eta_k$, $\eta_k \sim \calN(0, \sigma_\eta^2 I)$.}
        \For{$t = 0, ..., N-1$}
            \State{$\Qh_{t} = \mathsf{Proj}_{\mu}(\mathsf{LSTDQ}(\calD, K_t))$.}
            \State $K_{t+1} = G(\Qh_{t})$. [See \eqref{eq:T_def}.]
        \EndFor
        \State{{\bf return} $K_N$.}
    \end{algorithmic}
    \label{alg:lspi_offline_v1}
  \end{algorithm}
    \hspace{1cm}
\end{minipage}\quad
\begin{minipage}[t]{0.46\textwidth}
  \begin{algorithm}[H]
    \caption{$\mathsf{LSPIv2}$ for LQR}
    \begin{algorithmic}[1]
        \Require $K_0$: initial stabilizing controller,
        	\Statex ~~~~~~$N$: number of policy iterations,
        	\Statex ~~~~~~$T$: length of rollout,
        	\Statex ~~~~~~$\sigma_\eta^2$: exploration variance,
        	\Statex ~~~~~~$\mu$: lower eigenvalue bound.
        \For{$t = 0, ..., N-1$}
            \State \begin{varwidth}[t]{\linewidth}
            Collect $\calD_t = \{(x_k^{(t)}, u_k^{(t)}, x_{k+1}^{(t)})\}_{k=1}^{T}$, \par
            $u_k^{(t)} = K_0 x_k^{(t)} + \eta_k^{(t)}$,$\eta_k^{(t)} \sim \calN(0, \sigma_\eta^2 I)$.
            \end{varwidth} 
            \State{$\Qh_{t} = \mathsf{Proj}_{\mu}(\mathsf{LSTDQ}(\calD, K_t))$.}
            \State{$K_{t+1} = G(\Qh_{t})$.}
        \EndFor
        \State{{\bf return} $K_N$.}
    \end{algorithmic}
    \label{alg:lspi_offline_v2}
  \end{algorithm}
\end{minipage}

In Algorithms \ref{alg:lspi_offline_v1} and \ref{alg:lspi_offline_v2},
$\mathsf{Proj}_{\mu}(\cdot) = \arg\min_{X=X^\T : X \succeq \mu \cdot I} \norm{X - \cdot}_F$ is the
Euclidean projection onto the set of symmetric matrices
lower bounded by $\mu \cdot I$.
Furthermore, the map $G(\cdot)$ takes an $(n+d)\times (n+d)$ positive definite matrix and returns
a $d \times n$ matrix:
\begin{align}
	G\left( \bmattwo{Q_{11}}{Q_{12}}{Q_{12}^\T}{Q_{22}} \right) = - Q_{22}^{-1} Q_{12}^\T \:. \label{eq:T_def}
\end{align}

Algorithm~\ref{alg:lspi_offline_v1} corresponds to the 
version presented in \citet{lagoudakis03}, where all the data $\calD$ is collected up front
and is re-used in every iteration of LSTD-Q.
Algorithm~\ref{alg:lspi_offline_v2} is the one we will analyze in this paper,
where new data is collected for every iteration of LSTD-Q.
The modification made in Algorithm~\ref{alg:lspi_offline_v2} simplifies the analysis
by allowing the controller $K_t$ to be independent of the data $\calD_t$ in LSTD-Q.
We remark that this does \emph{not} require the system to be reset after
every iteration of LSTD-Q. We leave analyzing Algorithm~\ref{alg:lspi_offline_v1} to future work.

Before we state our main finite-sample guarantee for Algorithm~\ref{alg:lspi_offline_v2},
we review the notion of a (relative) value-function. 
Similarly to (relative) $Q$-functions,
the infinite horizon average-cost Bellman equation
states that the (relative) value function $V$ associated to a policy $\pi$ satisfies
the fixed-point equation:
\begin{align}
	\lambda + V(x) = c(x, \pi(x)) + \E_{x' \sim p(\cdot|x,\pi(x))}[V(x')] \:.
\end{align}
For a stabilizing policy $K$, it is well known that for LQR the value function $V(x) = x^\T V x$
with
\begin{align*}
	V = \dlyap(A+BK, S + K^\T R K) \:, \:\: \lambda = \ip{\sigma_w^2 I}{V} \:.
\end{align*}
Once again as we did for $Q$-functions, we slightly abuse notation and let $V$ denote
the value function and the matrix that parameterizes the value function.
Our main result for Algorithm~\ref{alg:lspi_offline_v2} appears in the following
theorem.
For simplicity, we will assume that $\norm{S} \geq 1$ and $\norm{R} \geq 1$.
\begin{thm}
\label{thm:lspi_estimation}
Fix a $\delta \in (0, 1)$. Let the initial policy $K_0$ input to Algorithm~\ref{alg:lspi_offline_v2}
stabilize $(A,B)$.
Suppose the initial state $x_0 \sim \calN(0, \Sigma_0)$
and that the excitation noise satisfies $\sigma_\eta \leq \sigma_w$.
Recall that the steady-state covariance of the trajectory $\{x_t\}$ is
\begin{align*}
	P_\infty = \dlyap((A+BK_0)^\T, \sigma_w^2 I + \sigma_\eta^2 BB^\T) \:.
\end{align*}
Let $V_0$ denote the value function associated to the initial policy $K_0$, and
$V_\star$ denote the value function associated to the optimal policy $K_\star$
for the LQR problem \eqref{eq:lqr_problem}.
Define the variables $\mu, L$ as:
\begin{align*}
    \mu &:= \min\{\lambda_{\min}(S), \lambda_{\min}(R)\} \:, \\
    L &:= \max\{\norm{S},\norm{R}\} + 2(\norm{A}^2+\norm{B}^2 + 1)\norm{V_0}_+ \:.
\end{align*}
Fix an $\varepsilon > 0$ that satisfies:
\begin{align}
    \varepsilon \leq 5 \left(\frac{L}{\mu}\right)^2 \min\left\{ 1, \frac{2\log(\norm{V_0}/\lambda_{\min}(V_\star))}{e}, \frac{\norm{V_\star}^2}{8\mu^2\log(\norm{V_0}/\lambda_{\min}(V_\star))} \right\}\:. \label{eq:lspi_Creq}
\end{align}
Suppose we run Algorithm~\ref{alg:lspi_offline_v2} for $N := N_0+1$ policy improvement iterations
where
\begin{align}
	N_0 := \bigceil{ (1+L/\mu) \log\left(\frac{2\log(\norm{V_0}/\lambda_{\min}(V_\star))}{\varepsilon}\right) } \:,
\end{align}
and we set the rollout length $T$ to satisfy:
\begin{align}
T \geq \Otilde(1) \max\bigg\{ &(n+d)^2, \nonumber \\
	&\frac{L^2}{(1-\mu/L)^2} \left(\frac{L}{\mu}\right)^{17} \frac{(n+d)^4}{\sigma_\eta^4} \sigma_w^2 (\norm{\Sigma_0} + \norm{P_\infty} + \sigma_\eta^2 \norm{B}^2), \nonumber \\
	&\frac{1}{\varepsilon^2} \frac{L^4}{(1-\mu/L)^2}\left( \frac{L}{\mu} \right)^{42} \frac{(n+d)^3}{\sigma_\eta^4}   \sigma_w^2 ( \norm{\Sigma_0} + \norm{P_\infty} + \sigma_\eta^2\norm{B}^2  ) \bigg\} \:. \label{eq:lspi_Cerr}
\end{align}
Then with probability $1-\delta$, we have that each policy $K_t$ for $t=1, ..., N$ stabilizes $(A,B)$ and
furthermore:
\begin{align*}
    \norm{K_N - K_\star} \leq \varepsilon \:.
\end{align*}
Here the $\Otilde(1)$ hides $\polylog(n,\tau,\norm{\Sigma_0},\norm{P_\infty},L/\mu,T/\delta, N_0, 1/\sigma_\eta)$ factors.
\end{thm}

Theorem~\ref{thm:lspi_estimation}
states roughly that $T \cdot N \leq \Otilde(\frac{(n+d)^3}{\varepsilon^2} \log(1/\varepsilon))$ samples are sufficient for LSPI
to recover a controller $K$ that is within $\varepsilon$ of the optimal $K_\star$.
That is, only $\log(1/\varepsilon)$ iterations of policy improvement are necessary,
and furthermore more iterations of policy improvement do not necessary help
due to the inherent statistical noise of estimating the $Q$-function for every policy $K_t$.
We note that the polynomial factor in $L/\mu$ is by no means optimal and was deliberately made
quite conservative in order to simplify the presentation of the bound. A sharper bound
can be recovered from our analysis techniques at the expense of a less concise expression.

It is worth taking a moment to compare Theorem~\ref{thm:lspi_estimation} to classical results
in the RL literature regarding approximate policy iteration.
For example, a well known result (c.f.~Theorem 7.1 of \citet{lagoudakis03}) states that
if LSTD-Q is able to return $Q$-function estimates with error $L_\infty$ bounded by $\varepsilon$
at every iteration,
then letting $\Qh_t$ denote the approximate $Q$-function at the $t$-th iteration of LSPI:
\begin{align*}
	\limsup_{t \to \infty} \:\norm{\Qh_t - Q_\star}_\infty \leq \frac{2\gamma \varepsilon}{(1-\gamma)^2} \:.
\end{align*}
Here, $\gamma$ is the discount factor of the MDP. Theorem~\ref{thm:lspi_estimation} is qualitatively similar
to this result in that we show roughly that $\varepsilon$ error in the $Q$-function estimate
translates to $\varepsilon$ error in the estimated policy. However, there are several fundamental
differences. First, our analysis does not rely on discounting to show contraction of the Bellman
operator. Instead, we use the $(\tau,\rho)$-stability of closed loop system to achieve this effect.
Second, our analysis does not rely on $L_\infty$ bounds on the estimated $Q$-function, which
are generally not possible to achieve with LQR since the $Q$-function is a quadratic function
and the states and inputs are not uniformly bounded. And finally, our analysis is non-asymptotic.

The proof of Theorem~\ref{thm:lspi_estimation} is given in Section~\ref{sec:lspi}, and
combines the estimation guarantee of
Theorem~\ref{thm:lstd_q_estimation} with a new analysis of policy iteration for LQR, which we
believe is of independent interest.
Our new policy iteration analysis combines the work of \citet{bertsekas17} on policy iteration in
infinite horizon average cost MDPs
with the contraction theory of \citet{lee08} for non-linear matrix equations.

\subsection{LSPI for Adaptive LQR}
\label{sec:results:adaptive}

We now turn our attention to the online, adaptive LQR problem as studied in \citet{abbasi11}.
In the adaptive LQR problem, the quantity of interest is the \emph{regret}, defined as:
\begin{align}
	\mathsf{Regret}(T) := \sum_{t=1}^{T} x_t^\T S x_t + u_t^\T R u_t - T \cdot J_\star \:. \label{eq:regret_def}
\end{align}
Here, the algorithm is penalized for the cost incurred from learning the optimal policy $K_\star$,
and must balance exploration (to better learn the optimal policy) versus exploitation (to reduce cost).
As mentioned previously, there are several known algorithms which achieve
$\Otilde(\sqrt{T})$ regret \cite{abbasi11,abeille18,ouyang17,cohen19,mania19}.
However, these algorithms operate in a \emph{model-based} manner, using the collected data to
build a confidence interval around the true dynamics $(A,B)$. On the other hand, the performance
of adaptive algorithms which are \emph{model-free} is less well understood. We use the results
of the previous section to give an adaptive model-free algorithm for LQR which achieves $\Otilde(T^{2/3})$
regret, which improves upon the $\Otilde(T^{2/3 + \varepsilon})$ regret (for $T \geq C^{1/\varepsilon}$)  achieved by the adaptive model-free algorithm of \citet{abbasi18}.
Our adaptive algorithm based on LSPI is shown in Algorithm~\ref{alg:lspi_online}.

\begin{center}
    \begin{algorithm}[htb]
    \caption{Online Adaptive Model-free LQR Algorithm}
    \begin{algorithmic}[1]
        \Require{Initial stabilizing controller $K^{(0)}$, number of epochs $E$, epoch multiplier $\Tmult$, lower eigenvalue bound $\mu$.}
        \For{$i = 0, ..., E-1$}
            \State{Set $T_i = \Tmult 2^i$.}
            \State{Set $\sigma_{\eta,i}^2 = \sigma_w^2\left(\frac{1}{2^i}\right)^{1/3}$.}
            \State{Set $K^{(i+1)} = \mathsf{LSPIv2}(K_0{=}K^{(i)}, N{=}\Otilde((i+1)\Gamma_\star/\mu), T{=}T_i, \sigma_\eta^2{=}\sigma_{\eta,i}^2)$.}
        \EndFor
    \end{algorithmic}
    \label{alg:lspi_online}
    \end{algorithm}
\end{center}

Using an analysis technique similar to that in \citet{dean18}, we
prove the following $\Otilde(T^{2/3})$ regret bound for Algorithm~\ref{alg:lspi_online}.
\begin{thm}
\label{thm:lspi_regret}
Fix a $\delta \in (0, 1)$.
Let the initial feedback $K^{(0)}$ stabilize $(A,B)$ and let $V^{(0)}$ denote its associated
value function.
Also let $K_\star$ denote the optimal LQR controller and let $V_\star$ denote the optimal value function.
Let $\Gamma_\star = 1 + \max\{\norm{A},\norm{B},\norm{V^{(0)}}, \norm{V_\star}, \norm{K^{(0)}}, \norm{K_\star},\norm{Q},\norm{R}\}$.
Suppose that $\Tmult$ is set to:
\begin{align*}
	\Tmult \geq \Otilde(1) \mathrm{poly}(\Gamma_\star,  n, d, 1/\lambda_{\min}(S)) \:.
\end{align*}
and suppose $\mu$ is set to $\mu = \min\{\lambda_{\min}(S), \lambda_{\min}(R)\}$.
With probability at least $1-\delta$, we have that the regret of Algorithm~\ref{alg:lspi_online}
satisfies:
\begin{align*}
	\mathsf{Regret}(T) \leq \Otilde(1) \mathrm{poly}(\Gamma_\star, n, d, 1/\lambda_{\min}(S)) T^{2/3} \:.
\end{align*}
\end{thm}

The proof of Theorem~\ref{thm:lspi_regret} appears in Section~\ref{sec:adaptive}.
We note that the regret scaling as $T^{2/3}$ in Theorem~\ref{thm:lspi_regret} is
due to the $1/\sigma_\eta^4$ dependence from LSTD-Q (c.f.~\eqref{eq:lstdq_T_bound}).
As mentioned previously, the existing LSTD-Q analysis from \citet{abbasi18}
yields a $1/\sigma_\eta^8$ dependence in LSTD-Q; using this $1/\sigma_\eta^8$ dependence in the analysis of
Algorithm~\ref{alg:lspi_online} would translate into $T^{4/5}$ regret.

\section{Related Work}

For model-based methods,
in the offline setting \citet{fiechter97} provided the first
PAC-learning bound for infinite horizon \emph{discounted} LQR using certainty equivalence (nominal) control.
Later, \citet{dean17} use tools from robust control to analyze a 
robust synthesis method for infinite horizon \emph{average cost} LQR, which is applicable in regimes of moderate
uncertainty when nominal control fails. \citet{mania19} show that 
certainty equivalence control actually provides a fast $\calO(\varepsilon^2)$ rate
of sub-optimality where $\varepsilon$ is the size of the parameter error,
unlike the $\calO(\varepsilon)$ sub-optimality guarantee of \cite{fiechter97,dean17}.
For the online adaptive setting, \cite{abbasi11,ibrahimi12,abeille18,mania19,cohen19} give $\Otilde(\sqrt{T})$ regret algorithms.
A key component of model-based algorithms is being able to quantify a 
confidence interval for the parameter estimate,
for which several recent works~\cite{simchowitz18,faradonbeh18,sarkar19} provide non-asymptotic
results.

Turning to model-free methods,
\citet{tu18a} study the behavior
of least-squares temporal difference (LSTD) for learning the \emph{discounted value} function
associated to a stabilizing policy. They evaluate the LSPI algorithm studied in this paper empirically, but do not provide any analysis.
In terms of policy optimization, most of the work has focused on
derivative-free random search methods \cite{fazel18,malik19}. 
\citet{tu18b} study a special family of LQR instances and 
characterize the asymptotic behavior of both model-based certainty equivalent control
versus policy gradients (REINFORCE), showing that policy gradients has
polynomially worse sample complexity.
Most related to our work is \citet{abbasi18}, who analyze a model-free
algorithm for adaptive LQR based on ideas from online convex optimization. 
LSTD-Q is a sub-routine of their algorithm, and their analysis
incurs a sub-optimal $1/\sigma_{\eta}^8$ dependence on the injected exploration noise, 
which we improve to $1/\sigma_{\eta}^4$ using self-normalized martingale inequalities~\cite{abbasi11b}. This improvement allows us to use a simple greedy exploration strategy
to obtain $T^{2/3}$ regret.
Finally, as mentioned earlier, the Ph.D. thesis of \citet{bradtke94} presents
an asymptotic consistency argument for approximate PI for discounted LQR in the noiseless
setting (i.e. $w_t = 0$ for all $t$). 

For the general function approximation setting in RL,
\citet{antos08} and \citet{lazaric12} analyze variants of LSPI for discounted MDPs where
the state space is compact and the action space finite. 
In \citet{lazaric12}, the policy is greedily updated
via an update operator that requires access to the underlying dynamics (and is therefore not implementable).
\citet{farahmand16} extend the results of \citet{lazaric12} to when the function
spaces considered are reproducing kernel Hilbert spaces.
\citet{zou19} give a finite-time analysis of both Q-learning and SARSA,
combining the asymptotic analysis of \citet{melo08} with the
finite-time analysis of TD-learning from \citet{bhandari18}. We note that checking
the required assumptions to apply the results of \citet{zou19} is non-trivial
(c.f.~Section 3.1, \cite{melo08}). We are un-aware of any 
non-asymptotic analysis of LSPI in the \emph{average cost} setting, which is
more difficult as the Bellman operator is no longer a contraction.

Finally, we remark that our LSPI analysis relies on understanding exact policy iteration
for LQR, which is closely related to the fixed-point Riccati recurrence (value iteration). 
An elegant analysis for value iteration is given by \citet{lincoln06}.
Recently, \citet{fazel18} show that exact policy iteration is a special case of 
Gauss-Newton and prove linear convergence results.
Our analysis, on the other hand, is based on combining the fixed-point theory from \citet{lee08}
with recent work on policy iteration for average cost problems from \citet{bertsekas17}. 

\def \basefigwidth{0.49}
\def \singlefigwidth{0.6}

\section{Experiments}
\label{sec:experiments}

In this section, we evaluate LSPI in both the non-adaptive offline setting
(Section~\ref{sec:results:lspi_offline}) as well as the adaptive online setting
(Section~\ref{sec:results:adaptive}).  Section~\ref{sec:app:experiments}
contains more details about both the algorithms we compare to as well as our
experimental methodology.

We first look at the performance of LSPI in the non-adaptive, offline setting.
Here, we compare LSPI to other popular model-free methods, and
the model-based certainty equivalence (nominal) controller (c.f.~\cite{mania19}).
For model-free, we look at policy gradients (REINFORCE) (c.f.~\cite{williams92})
and derivative-free optimization (c.f.~\cite{mania18,malik19,nesterov17}).
We consider the LQR instance $(A,B,S,R)$ with
\begin{align*}
A = \begin{bmatrix} 0.95 & 0.01 & 0 \\
	0.01 & 0.95 & 0.01 \\
	0 & 0.01 & 0.95
	\end{bmatrix} \:, \:\:
B = \begin{bmatrix}
	1 & 0.1 \\
	0 & 0.1 \\
	0 & 0.1
	\end{bmatrix} \:,\:\:
S = I_3 \:, \:\: R = I_2 \:.
\end{align*}
We choose an LQR problem where the $A$ matrix is stable, since
the model-free methods we consider need to be seeded with an initial
stabilizing controller; using a stable $A$ allows us to start
at $K_0 = 0_{2 \times 3}$.
We fix the process noise $\sigma_w = 1$.
The model-based nominal method learns $(A, B)$ using
least-squares,
exciting the system with Gaussian inputs $u_t$ with variance $\sigma_u = 1$.

For policy gradients and derivative-free optimization, we use the
projected stochastic gradient descent (SGD) method with a constant step size $\mu$ as
the optimization procedure.
For policy iteration,
we evaluate both $\mathsf{LSPIv1}$ (Algorithm~\ref{alg:lspi_offline_v1})
and $\mathsf{LSPIv2}$ (Algorithm~\ref{alg:lspi_offline_v2}).
For every iteration of LSTD-Q, we project
the resulting $Q$-function parameter matrix onto the
set $\{ Q : Q \succeq \gamma I \}$
with $\gamma = \min\{\lambda_{\min}(S), \lambda_{\min}(R) \}$.
For $\mathsf{LSPIv1}$, we choose
$N = 15$ by picking the $N \in [5, 10, 15]$ which results in the best performance after
$T = 10^6$ timesteps.
For $\mathsf{LSPIv2}$,
we set $(N, T) = (3, 333333)$
which yields the lowest cost
over the grid $N \in [1, 2, 3, 4, 5, 6, 7]$ and $T$ such that $NT = 10^{6}$.

\begin{figure}[h!]
\centering
\includegraphics[width=0.65\columnwidth]{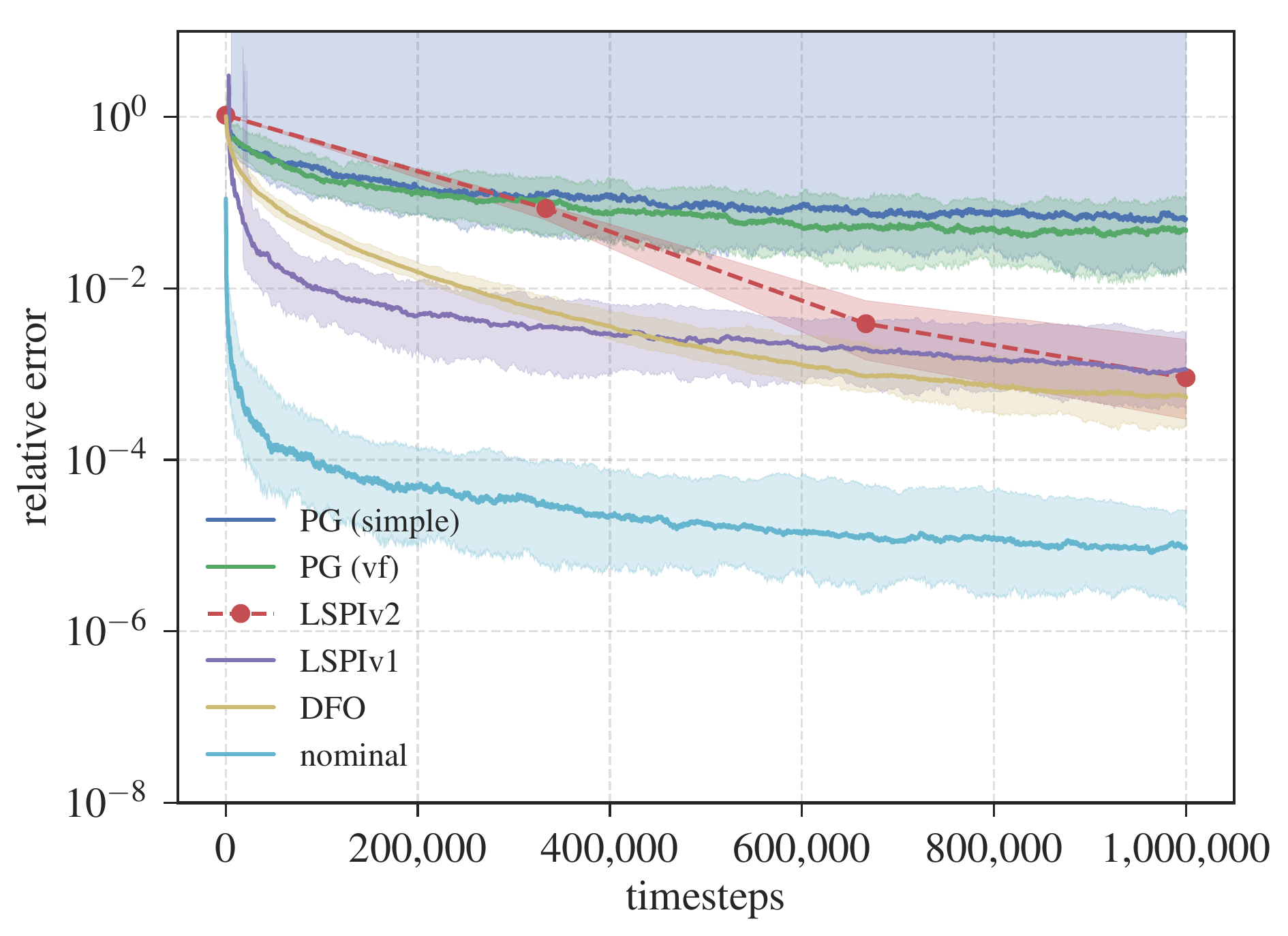}
\caption{Plot of non-adaptive performance.
The shaded regions represent the lower 10th and upper 90th percentile over $100$ trials, and
the solid line represents the median performance.
Here, PG (simple) is policy gradients with the simple baseline,
PG (vf) is policy gradients with the value function baseline,
LSPIv2 is Algorithm~\ref{alg:lspi_offline_v2},
LSPIv1 is Algorithm~\ref{alg:lspi_offline_v1},
and DFO is derivative-free optimization.}
\label{fig:offline}
\end{figure}

Figure~\ref{fig:offline} contains the results of our non-adaptive evaluation.
In Figure~\ref{fig:offline},
we plot the relative error $(J(\Kh) - J_\star)/J_\star$ versus the number of timesteps.
We see that the model-based certainty equivalence (nominal) method is more sample efficient than the other model-free methods considered.
We also see that the value function baseline is able to dramatically reduce the variance
of the policy gradient estimator compared to the simple baseline.
The DFO method performs the best out of all the model-free methods considered on this example
after $10^6$ timesteps, although the performance of policy iteration is
comparable.

Next, we compare the performance of LSPI in the adaptive setting. We compare LSPI against the model-free linear quadratic control (MFLQ) algorithm of \citet{abbasi18},
the certainty equivalence (nominal) controller (c.f.~\cite{dean18}), and the optimal controller.
We set the process noise $\sigma_w = 1$, and 
consider the example of \citet{dean17}:
\begin{align*}
A = \begin{bmatrix} 1.01 & 0.01 & 0 \\
	0.01 & 1.01 & 0.01 \\
	0 & 0.01 & 1.01
	\end{bmatrix} \:, \:\:
B = I_3 \:, \:\: S = 10 \cdot I_3 \:, \:\: R = I_3 \:.
\end{align*}

\begin{figure}[h!]
\centering
\begin{subfigure}[b]{\basefigwidth\textwidth}
\caption{\small Regret Comparison}
\centerline{\includegraphics[width=\columnwidth]{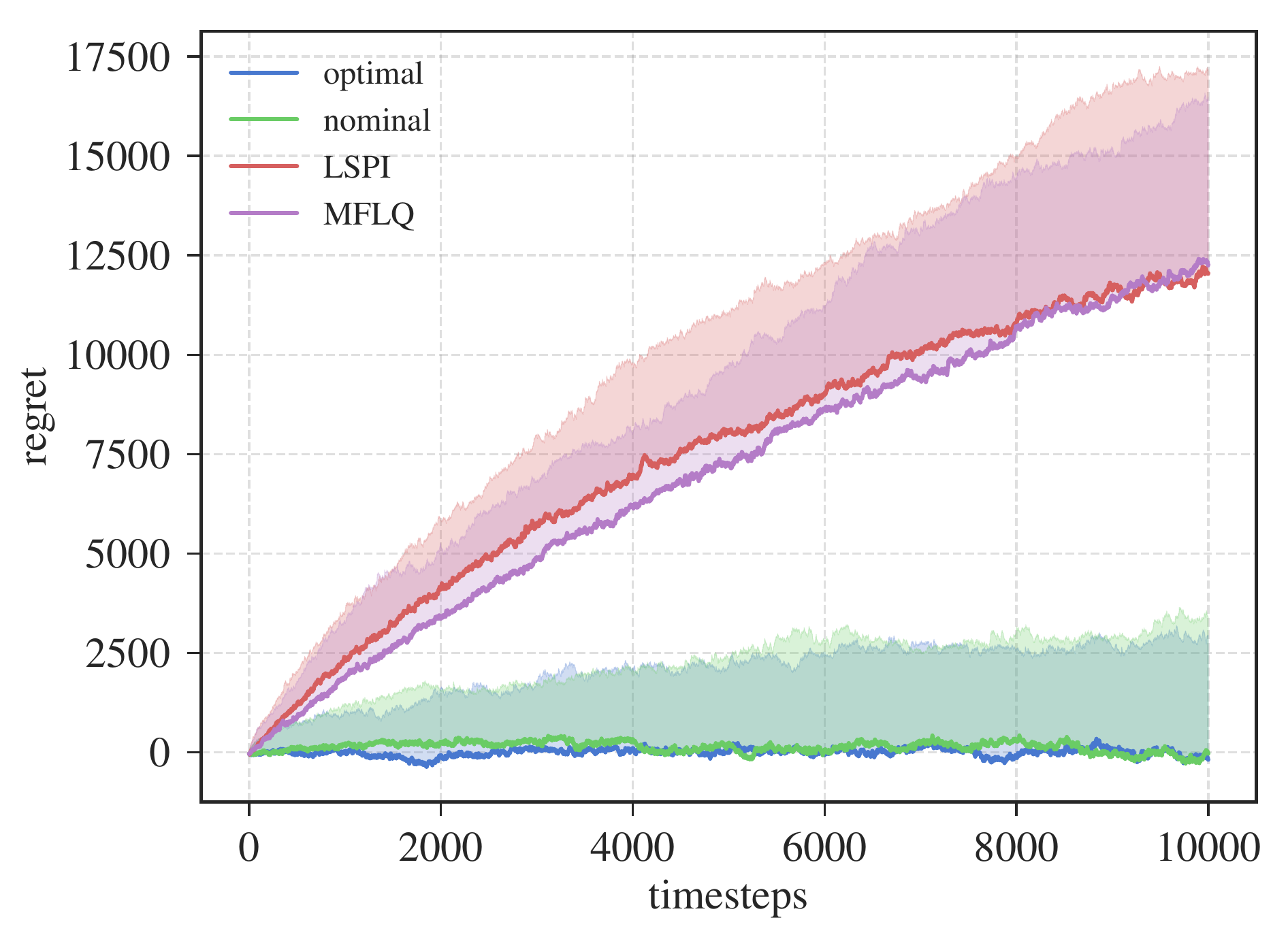}}
\label{fig:regret_comparison}
\end{subfigure}
\begin{subfigure}[b]{\basefigwidth\textwidth}
\caption{\small Cost Comparison}
\centerline{\includegraphics[width=\columnwidth]{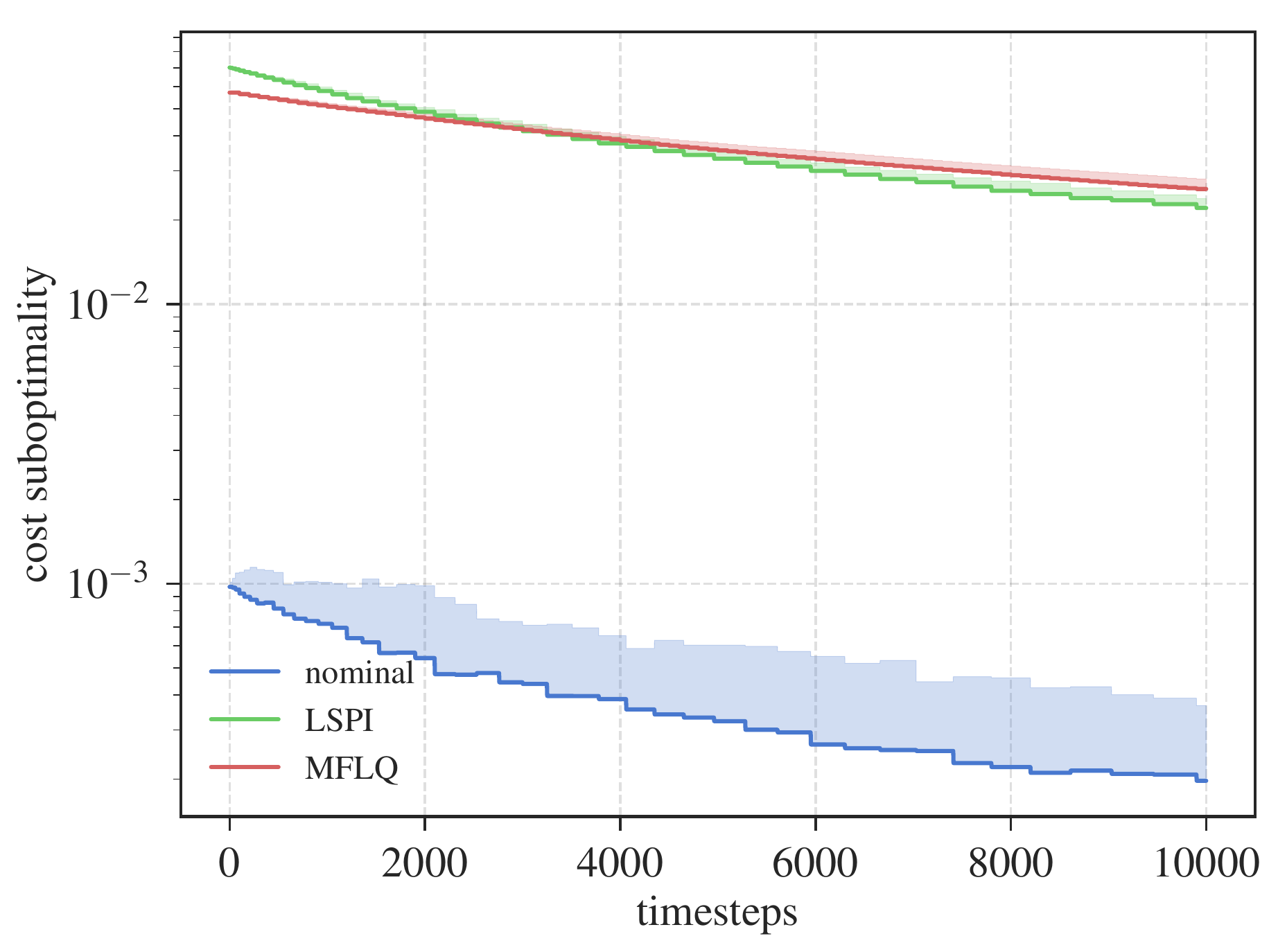}}
\label{fig:cost_comparison}
\end{subfigure}
\caption{
Plot of adaptive performance.
The shaded regions represent the median to upper 90th percentile over $100$ trials.
Here, LSPI is Algorithm~\ref{alg:lspi_online}
using LSPIv1, MFLQ is from \citet{abbasi18}, nominal is
the $\varepsilon$-greedy adaptive certainty equivalent controller (c.f.~\cite{dean18}),
and optimal has access to the true dynamics.
\textbf{(a)} Plot of regret versus time.
\textbf{(b)} Plot of the cost sub-optimality versus time.}
\label{fig:adaptive}
\end{figure}

Figure~\ref{fig:adaptive} shows the results of these experiments.
In Figure~\ref{fig:regret_comparison}, we plot the regret (c.f. Equation~\ref{eq:regret_def}) versus
the number of timesteps.
We see that LSPI and MFLQ both perform similarly with MFLQ slightly outperforming LSPI. We also note that the model-based nominal controller performs significantly better than both LSPI and MFLQ,
which is consistent with the experiments of \citet{abbasi18}.
In Figure~\ref{fig:cost_comparison}, we plot the relative cost $(J(\Kh) - J_\star)/J_\star$ versus
the number of timesteps. This quantity represents the sub-optimality incurred if further exploration ceases and the current controller is played indefinitely. Here, we see again that LSPI and MFLQ are both
comparable, but both are outperformed by nominal control.


\section{Conclusion}

We studied the sample complexity of approximate PI on LQR, showing that order $(n+d)^3 \varepsilon^{-2} \log(1/\varepsilon)$ samples
are sufficient to estimate a controller that is within $\varepsilon$ of the optimal.
We also show how to turn this offline method into an adaptive LQR method with
$T^{2/3}$ regret.
Several questions remain open with our work. The first is if
policy iteration is able to achieve $T^{1/2}$ regret,
which is possible with other model-based methods.
The second is whether or not model-free methods provide advantages
in situations of partial observability for LQ control.
Finally, an asymptotic analysis of LSPI, in the spirit of \citet{tu18b},
is of interest in order to clarify which parts of our analysis are sub-optimal
due to the techniques we use versus are inherent in the algorithm. 

\section*{Acknowledgments}
We thank the authors of \citet{abbasi18} for providing us with an implementation of their model-free LQ algorithm.
ST is supported by a Google PhD fellowship.
This work is generously supported in part by ONR awards N00014-17-1-2191, N00014-17-1-2401, and N00014-18-1-2833, the DARPA Assured Autonomy (FA8750-18-C-0101) and Lagrange (W911NF-16-1-0552) programs, a Siemens Futuremakers Fellowship, and an Amazon AWS AI Research Award.

{
\small
\bibliography{paper}
}

\appendix


\section{Analysis for LSTD-Q}
\label{sec:lstdq}

We fix a trajectory $\{(x_t, u_t, x_{t+1})\}_{t=1}^{T}$.
Recall that we are interested in finding the $Q$ function for a given policy $\Keval$,
and we have defined the vectors:
\begin{align*}
	\phi_t &= \phi(x_t, u_t) \:, \:\: \psi_t = \phi(x_t, \Keval x_t) \:, \\
    f &= \svec\left(\sigma_w^2 \cvectwo{I}{\Keval}\cvectwo{I}{\Keval}^\T\right) \:, \:\: c_t = x_t^\T S x_t + u_t^\T R u_t \:.
\end{align*}
Also recall that the input sequence $u_t$ being played is given by $u_t = \Kplay x_t + \eta_t$,
with $\eta_t \sim \calN(0, \sigma_\eta^2 I)$. Both policies $\Keval$ and $\Kplay$ are assumed to stabilize
$(A,B)$. Because of stability, we have that $P_t$ converges to a limit $P_\infty = \dlyap((A+B\Kplay)^\T,\sigma_w^2 I + \sigma_\eta^2 BB^\T)$,
where $P_t$ is:
\begin{align*}
    P_t := \sum_{k=0}^{t-1} (A+B\Kplay)^k (\sigma_w^2 I + \sigma_\eta BB^\T) ((A+B\Kplay)^\T)^k \:.
\end{align*}
The covariance of $x_t$ for $t \geq 1$ is:
\begin{align*}
    \Cov(x_t) = \Sigma_t := P_t + (A+B\Kplay)^t \Sigma_0 ((A+B\Kplay)^\T)^t \:.
\end{align*}
We define the following data matrices:
\begin{align*}
    \Phi = \begin{bmatrix}
        -\phi_1^\T- \\
        \vdots \\
        -\phi_T^\T-
    \end{bmatrix} \:, \:\:
    \Psi_+ =  \begin{bmatrix}
        -\psi_2^\T- \\
        \vdots \\
        -\psi_{T+1}^\T-
    \end{bmatrix} \:, \:\:
    c = (c_1, ..., c_T)^\T \:, \:\:
    F = \begin{bmatrix}
        -f^\T- \\
        \vdots \\
        -f^\T-
    \end{bmatrix} \:.
\end{align*}
With this notation, the LSTD-Q estimator is:
\begin{align*}
    \qh = \left(\Phi^\T (\Phi - \Psi_+ + F)\right)^{\dag} \Phi^\T c \:.
\end{align*}
Next, let $\Xi$ be the matrix:
\begin{align*}
    \Xi = \begin{bmatrix}
        -\E[\phi(x_2, \Keval x_2)|x_1, u_1]^\T- \\
        \vdots \\
        -\E[\phi(x_{T+1}, \Keval x_{T+1})|x_T, u_T]^\T-
    \end{bmatrix}  \:.
\end{align*}
For what follows, we let the notation $\otimes_s$ denote the
\emph{symmetric} Kronecker product. See \citet{schacke13} for more details.
The following lemma gives us a starting point for analysis.
It is based on Lemma 4.1 of \citet{abbasi18}.
Recall that $q = \svec(Q)$ and $Q$ is the matrix which parameterizes
the $Q$-function for $\Keval$.
\begin{lem}[Lemma 4.1, \cite{abbasi18}]
\label{lem:basic_inequality}
Let $L := \cvectwo{I}{\Keval} \rvectwo{A}{B}$.
Suppose that $\Phi$ has full column rank, and that 
\begin{align*}
	\frac{\norm{(\Phi^\T \Phi)^{-1/2} \Phi^\T(\Xi-\Psi_+)}}{\smin(\Phi) \smin(I - L \otimes_s L)} \leq 1/2 \:.
\end{align*}
Then we have:
\begin{align}
    \norm{\qh - q} \leq 2 \frac{\norm{(\Phi^\T \Phi)^{-1/2}\Phi^\T(\Xi-\Psi_+)q}}{\smin(\Phi) \smin(I - L \otimes_s L)} \:. \label{eq:main_eq}
\end{align}
\end{lem}
\begin{proof}
By the Bellman equation \eqref{eq:bellman_q}, we have the identity:
\begin{align*}
    \Phi q = c + (\Xi - F) q
\end{align*}
By the definition of $\qh$, we have the identity:
\begin{align*}
    \Phi \qh = P_{\Phi}( c + (\Psi_+ - F) \qh) \:,
\end{align*}
where $P_\Phi = \Phi (\Phi^\T \Phi)^{-1} \Phi^\T$ is the orthogonal projector onto the columns of $\Phi$.
Combining these two identities gives us:
\begin{align*}
    P_\Phi(\Phi - \Xi + F) (q - \qh) = P_\Phi (\Xi - \Psi_+) \qh \:.
\end{align*}
Next, the $i$-th row of
$\Phi - \Xi + F$ is:
\begin{align*}
    &\svec\left(  \cvectwo{x_i}{u_i} \cvectwo{x_i}{u_i}^\T - \E\left[ \cvectwo{I}{\Keval} \tilde{x} \tilde{x}^\T\cvectwo{I}{\Keval}^\T \:\bigg|\: x_i, u_i\right] + \sigma_w^2 \cvectwo{I}{\Keval}\cvectwo{I}{\Keval}^\T \right) \\
    &\qquad= \svec\left( \cvectwo{x_i}{u_i} \cvectwo{x_i}{u_i}^\T -  L \cvectwo{x_i}{u_i} \cvectwo{x_i}{u_i}^\T L^\T   \right) \\
    &\qquad= (I - L \otimes_s L ) \phi(x_i, u_i) \:.
\end{align*}
Therefore, $\Phi - \Xi + F = \Phi (I - L \otimes_s L)^\T$.
Combining with the above identity:
\begin{align*}
    \Phi (I - L \otimes_s L)^\T (q - \qh) = P_\Phi (\Xi - \Psi_+) \qh \:.
\end{align*}
Because $\Phi$ has full column rank, this identity implies that:
\begin{align*}
    (I - L \otimes_s L)^\T (q - \qh) = (\Phi^\T \Phi)^{-1} \Phi^\T (\Xi - \Psi_+) \qh \:.
\end{align*}
Using the inequalities:
\begin{align*}
    \norm{(I - L \otimes_s L)^\T (q - \qh)} &\geq \smin((I - L \otimes_s L)) \norm{q-\qh} \:, \\
    (\Phi^\T \Phi)^{-1} \Phi^\T (\Xi - \Psi_+) \qh &\leq \frac{\norm{ (\Phi^\T \Phi)^{-1/2} \Phi^\T(\Xi - \Psi_+) \qh}}{\lambda_{\min}( (\Phi^\T \Phi)^{-1/2} )} = \frac{\norm{ (\Phi^\T \Phi)^{-1/2} \Phi^\T(\Xi - \Psi_+) \qh}}{\smin(\Phi)} \:,
\end{align*}
we obtain:
\begin{align*}
    \norm{q-\qh} \leq \frac{\norm{ (\Phi^\T \Phi)^{-1/2} \Phi^\T(\Xi - \Psi_+) \qh}}{\smin(\Phi) \smin(I - L \otimes_s L)} \:.
\end{align*}
Next, let $\Delta = q - \qh$. By triangle inequality:
\begin{align*}
    \norm{\Delta} \leq \frac{\norm{ (\Phi^\T \Phi)^{-1/2} \Phi^\T(\Xi - \Psi_+)}\norm{\Delta}}{\smin(\Phi) \smin(I - L \otimes_s L)} + \frac{\norm{ (\Phi^\T \Phi)^{-1/2} \Phi^\T(\Xi - \Psi_+)q}}{\smin(\Phi) \smin(I - L \otimes_s L)} \:.
\end{align*}
The claim now follows.
\end{proof}

In order to apply Lemma~\ref{lem:basic_inequality}, we first
bound the minimum singular value $\smin(\Phi)$. We do this using the small-ball
argument of \citet{simchowitz18}.
\begin{defn}[Definition 2.1, \cite{simchowitz18}]
\label{def:bmsb}
Let $\{Z_t\}$ be a real-valued stochastic process that is adapted to $\{\calF_t\}$.
The process $\{Z_t\}$ satisfies the $(k,\nu,p)$ block martingale small-ball (BMSB) condition if
for any $j\geq 0$ we have that:
\begin{align*}
	\frac{1}{k} \sum_{i=1}^{k} \Pr( \abs{Z_{j+i}} \geq \nu | \calF_j ) \geq p \:\: \text{a.s.} 
\end{align*}
\end{defn}
With the block martingale small-ball definition in place, we
now show that the process $\ip{\phi_t}{y}$ satisfies this condition for any fixed unit vector $y$.
\begin{prop}
\label{prop:bmsb}
Given an arbitrary vector $y \in \calS^{n+d-1}$, 
define the process $Z_{t} := \langle \phi_t, y \rangle$,
the filtration $\mathcal{F}_t := \sigma(\{u_i, w_{i - 1}\}_{i=0}^t)$, 
and matrix $C := \begin{bmatrix}I & 0\\ \Kplay & I\end{bmatrix}\begin{bmatrix}\sigma_w I & 0\\ 0 & \sigma_\eta I\end{bmatrix}$. Then $(Z_t)_{t\ge 1}$ satisfies the $(1, \smin^2(C), 1/324)$ block martingale small-ball (BMSB) condition from Definition~\ref{def:bmsb}. 
That is, almost surely, we have:
\[\Pr(|Z_{t+1}|\ge \smin^2(C)|\mathcal{F}_t) \ge 1/324.\]
\end{prop}
\begin{proof}
Let $Y := \smat(y)$ and $\mu_t := A x_t + B u_t$.
We have that:
\begin{align*}
    \cvectwo{x_{t+1}}{u_{t+1}} = \cvectwo{I}{\Kplay} \mu_t + \bmattwo{I}{0}{\Kplay}{I} \cvectwo{w_t}{\eta_{t+1}} \:.
\end{align*}
Therefore:
\begin{align*}
    \ip{\phi_{t+1}}{y} &= \cvectwo{x_{t+1}}{u_{t+1}}^\T Y\cvectwo{x_{t+1}}{u_{t+1}} \\
    &= \left( \cvectwo{I}{\Kplay} \mu_t + \bmattwo{I}{0}{\Kplay}{I} \cvectwo{w_t}{\eta_{t+1}} \right)^\T Y\left( \cvectwo{I}{\Kplay} \mu_t + \bmattwo{I}{0}{\Kplay}{I} \cvectwo{w_t}{\eta_{t+1}} \right) \:,
\end{align*}
which is clearly a Gaussian polynomial of degree $2$ given $\mathcal{F}_t$. Hence by Gaussian hyper-contractivity results (see e.g. \cite{bogachev15}), we have that almost surely:
\[\E[|Z_{t+1}|^4|\mathcal{F}_t] \le 81\E[|Z_{t+1}|^2|\mathcal{F}_t]^2.\]
Hence we can invoke the Paley-Zygmund inequality to conclude that for any $\theta \in (0, 1)$, almost surely we have:
\[\Pr(|Z_{t+1}| \ge \sqrt{\theta \E[|Z_{t+1}|^2|\mathcal{F}_t]}|\mathcal{F}_t) \ge (1 - \theta)^2\frac{\E[|Z_{t+1}|^2|\mathcal{F}_t]^2}{\E[|Z_{t+1}|^4|\mathcal{F}_t]} \ge \frac{(1 - \theta)^2}{81}.\]

We now state an useful proposition.
\begin{prop}
\label{prop:quad_fourth_moment}
Let $\mu, C, Y$ be fixed and $g \sim \calN(0, I)$. We have that:
\begin{align*}
    \E[ ((\mu+Cg)^\T Y (\mu+Cg))^2 ] \geq 2 \norm{C^\T Y C}_F^2 \:.
\end{align*}
\end{prop}
\begin{proof}
Let $Z := (\mu+Cg)^\T Y (\mu+Cg)$.
We know that $\E[Z^2] \geq \E[ (Z-\E[Z])^2 ]$.
A quick computation yields that $\E[Z] = \mu^\T Y \mu + \Tr(C^\T Y C)$.
Hence
\begin{align*}
    Z - \E[Z] = g^\T C^\T Y C g - \Tr(C^\T Y C) + 2 \mu^\T Y C g \:.
\end{align*}
Therefore,
\begin{align*}
    \E[(Z-\E[Z])^2] \geq \E[ (g^\T C^\T Y C g - \Tr(C^\T Y C))^2 ] = 2 \norm{C^\T Y C}_F^2 \:.
\end{align*}
\end{proof}
Invoking Proposition~\ref{prop:quad_fourth_moment} and using 
basic properties of the Kronecker product, we have that:
\begin{align*}
    \E[ Z_{t+1}^2 | \calF_t ] \geq 2 \norm{C^\T Y C}_F^2 = 2 \norm{(C^\T \otimes C^\T) y}^2 \geq 2 \smin^2(C^\T \otimes C^\T) = 2 \smin^4(C) \:.
\end{align*}

The claim now follows by setting $\theta=1/2$.
\end{proof}

With the BMSB bound in place, we can now utilize Proposition 2.5 of \citet{simchowitz18}
to obtain the following lower bound on the minimum singular value $\smin(\Phi)$.
\begin{prop}\label{prop:singularbound}
Fix $\delta \in (0, 1)$. Suppose that $\sigma_\eta \leq \sigma_w$,
and that $T$ exceeds:
\begin{align}
    T \geq 324^2 \cdot 8 \left( (n+d)^2 \log\left(1+ \frac{20736\sqrt{3}}{\sqrt{\delta}} \frac{ (1+\norm{\Kplay}^2)^2 (\tau^2 \rho^2 n \norm{\Sigma_0} + \Tr(P_\infty))  }{\sigma_\eta^2}  \right) + \log(2/\delta) \right) \:. \label{eq:T_req}
\end{align}
Suppose also that $A+B\Kplay$ is $(\tau,\rho)$-stable.
Then we have with probability at least $1-\delta$,
\begin{align*}
    \smin(\Phi) \geq \frac{\sigma_\eta^2}{1296\sqrt{8}} \frac{1}{1+\norm{\Kplay}^2} \sqrt{T} \:.
\end{align*}
We also have with probability at least $1-\delta$,
\begin{align*}
    \norm{\Phi^\T \Phi} \leq \frac{12T}{\delta} (1+\norm{\Kplay}^2)^2 (\tau^2\rho^2n \norm{\Sigma_0} + \Tr(P_\infty))^2 \:.
\end{align*}
\end{prop}
\begin{proof}
We first compute a crude upper bound on $\norm{\Phi}$ using Markov's inequality:
\begin{align*}
    \Pr( \norm{\Phi}^2 \geq t^2 ) = \frac{\E[ \lambda_{\max}(\Phi^\T \Phi) ]}{t^2} \leq \frac{\Tr(\E[\Phi^\T \Phi])}{t^2} \:.
\end{align*}
Now we upper bound $\E[\norm{\phi_t}^2]$.
Letting $z_t = (x_t, u_t)$, we have that $\E[\norm{\phi_t}^2] = \E[\norm{z_t}^4] \leq 3 (\E[\norm{z_t}^2])^2$.
We now bound $\E[ \norm{z_t}^2 ] \leq (1+\norm{\Kplay}^2) \Tr(\Sigma_t) + \sigma_\eta^2 d$, and therefore:
\begin{align*}
    \sqrt{\E[\norm{\phi_t}^2]} &\leq \sqrt{3} ((1+\norm{\Kplay}^2) \Tr(\Sigma_t) + \sigma_\eta^2 d) \\
    &\leq \sqrt{3}((1+\norm{\Kplay}^2)(\tau^2\rho^2 n \norm{\Sigma_0} + \Tr(P_\infty)) + \sigma_\eta^2 d) \\
    &\leq 2\sqrt{3} (1+\norm{\Kplay}^2)(\tau^2\rho^2 n \norm{\Sigma_0} + \Tr(P_\infty)) \:.
\end{align*}
Above, the last inequality holds because $\sigma_\eta^2 d \leq \sigma_w^2 n \leq \Tr(P_\infty)$.
Therefore, we have from Markov's inequality:
\begin{align*}
    \Pr\left(\norm{\Phi} \geq \frac{\sqrt{T}}{\sqrt{\delta}} 2\sqrt{3}  (1+\norm{\Kplay}^2) (\tau^2\rho^2 n \norm{\Sigma_0} + \Tr(P_\infty)) \right) \leq \delta \:.
\end{align*}
Fix an $\varepsilon > 0$, and let $\calN(\varepsilon)$ denote an $\varepsilon$-net of the unit sphere
$\calS^{(n+d)(n+d+1)/2-1}$.
Next, by Proposition~2.5 of \citet{simchowitz18} and a union bound over $\calN(\varepsilon)$:
\begin{align*}
    \Pr\left( \min_{v \in \calN(\varepsilon)} \norm{\Phi v} \geq \frac{\smin^2(C)}{324\sqrt{8}}\sqrt{T} \right) \geq 1 - (1+2/\varepsilon)^{(n+d)^2} e^{-\frac{T}{324^2 \cdot 8}} \:.
\end{align*}
Now set
\begin{align*}
    \varepsilon = \frac{\sqrt{\delta}}{5184\sqrt{3}} \frac{\smin^2(C)}{ (1+\norm{\Kplay}^2) (\tau^2\rho^2 n \norm{\Sigma_0} + \Tr(P_\infty))} \:,
\end{align*}
and observe that as long as $T$ exceeds:
\begin{align*}
    T \geq 324^2 \cdot 8 \left( (n+d)^2 \log\left(1+ \frac{10368 \sqrt{3}}{\sqrt{\delta}} \frac{(1+\norm{\Kplay}^2)(\tau^2\rho^2 n \norm{\Sigma_0} + \Tr(P_\infty))}{\smin^2(C)}  \right) + \log(2/\delta) \right) \:,
\end{align*}
we have that
$\Pr\left( \min_{v \in \calN(\varepsilon)} \norm{\Phi v} \geq \frac{\smin^2(C)}{324\sqrt{8}}\sqrt{T} \right) \geq 1-\delta/2$.
To conclude, observe that:
\begin{align*}
    \smin(\Phi) = \inf_{\norm{v}=1} \norm{\Phi v} \geq \min_{v \in \calN(\varepsilon)} \norm{\Phi v} - \norm{\Phi}\varepsilon \:,
\end{align*}
and union bound over the two events. 
To conclude the proof, note that Lemma F.6 in \citet{dean18}
yields that $\smin^2(C) \geq \frac{\sigma_\eta^2}{2} \frac{1}{1+\norm{\Kplay}^2}$ since $\sigma_\eta \leq \sigma_w$.
\end{proof}

We now turn our attention to upper bounding the self-normalized martingale terms:
\begin{align*}
	\norm{(\Phi^\T \Phi)^{-1} \Phi^\T (\Xi - \Psi_+)} \:\: \text{and} \:\:
	\norm{(\Phi^\T \Phi)^{-1} \Phi^\T (\Xi - \Psi_+)q} \:.
\end{align*}
Our main tool here will be the self-normalized tail bounds of \citet{abbasi11b}.
\begin{lem}[Corollary 1, \cite{abbasi11b}]
\label{lem:self_normalized}
Let $\{\calF_t\}$ be a filtration.
Let $\{x_t\}$ be a $\R^{d_1}$ process that is adapted to $\{ \calF_t \}$
and let $\{w_t\}$ be a $\R^{d_2}$ martingale difference sequence that is adapted to $\{ \calF_t \}$.
Let $V$ be a fixed positive definite $d_1 \times d_1$ matrix and define:
\begin{align*}
    \bar{V}_t = V + \sum_{s=1}^{t} x_sx_s^\T \:, \:\: S_t = \sum_{s=1}^{t} x_s w_{s+1}^\T \:.
\end{align*}
\begin{enumerate}[(a)]
\item 
Suppose for any fixed unit $h \in \R^{d_2}$ we have that $\ip{w_t}{h}$ is conditionally $R$-sub-Gaussian, that is:
\begin{align*}
    \forall \lambda \in \R, t \geq 1 \:, \:\: \E[ e^{\lambda \ip{w_{t+1}}{h}} | \calF_t ] \leq e^{\frac{\lambda^2 R^2}{2}} \:.
\end{align*}

We have that with probability at least $1-\delta$, for all $t \geq 1$,
\begin{align*}
    \norm{ \bar{V}_t^{-1/2} S_t }^2 \leq 8 R^2 \left(d_2 \log{5} + \log\left( \frac{\det(\bar{V}_t)^{1/2} \det(V)^{-1/2}}{\delta} \right) \right) \:.
\end{align*}
\item 
Now suppose that $\bar{\delta}$ satisfies the condition:
\begin{align*}
    \sum_{s=2}^{T+1} \Pr( \norm{w_{s}} > R ) \leq \bar{\delta} \:.
\end{align*}
Then with probability at least $1-\delta-\bar{\delta}$,
for all $1 \leq t \leq T$,
\begin{align*}
    \norm{ \bar{V}_t^{-1/2} S_t }^2 \leq 32 R^2 \left(d_2 \log{5} + \log\left( \frac{\det(\bar{V}_t)^{1/2} \det(V)^{-1/2}}{\delta} \right) \right) \:.
\end{align*}
\end{enumerate}
\end{lem}
\begin{proof}
Fix a unit $h \in \R^{d_2}$.
By Corollary 1 of \citet{abbasi11b}, we have with probability at least $1-\delta$,
\begin{align*}
    \norm{ \bar{V}_t^{-1/2} S_t h }^2 \leq 2 R^2 \log\left( \frac{\det(\bar{V}_t)^{1/2} \det(V)^{-1/2}}{\delta} \right) \:, \:\: 1 \leq t \leq T \:.
\end{align*}
A standard covering argument yields that:
\begin{align*}
    \norm{ \bar{V}_t^{-1/2} S_t }^2 \leq 4 \max_{h \in \calN(1/2)} \norm{ \bar{V}_t^{-1/2} S_t h }^2 \:.
\end{align*}
Union bounding over $\calN(1/2)$, we obtain that:
\begin{align*}
    \norm{ \bar{V}_t^{-1/2} S_t }^2 &\leq 8 R^2 \log\left( 5^{d_2} \frac{\det(\bar{V}_t)^{1/2} \det(V)^{-1/2}}{\delta} \right) \\
    &= 8R^2 \left(d_2 \log{5} + \log\left( \frac{\det(\bar{V}_t)^{1/2} \det(V)^{-1/2}}{\delta} \right) \right) \:.
\end{align*}
This yields (a).

For (b), we use a simple stopping time argument to handle truncation.
Define the stopping time $\tau := \inf\{ t \geq 1 : \norm{w_t} > R \}$
and the truncated process $\tilde{w}_t := w_t \ind_{\tau \geq t}$.
Because $\tau$ is a stopping time, this truncated process $\{\tilde{w}_t\}$
remains a martingale difference sequence.
Define $Z_t = \sum_{s=1}^{t} x_s \tilde{w}_{s+1}^\T$.
For any $\ell > 0$ we observe that:
\begin{align*}
    &\Pr( \exists 1 \leq t \leq T : \norm{ \bar{V}_t^{-1/2} S_t } > \ell ) \\
    &\leq \Pr( \{\exists 1 \leq t \leq T : \norm{ \bar{V}_t^{-1/2} S_t } > \ell \} \cap \{ \tau > T+1 \} ) + \Pr(\tau \leq T+1) \\
    &= \Pr( \{\exists 1 \leq t \leq T : \norm{ \bar{V}_t^{-1/2} Z_t } > \ell \} \cap \{ \tau > T+1 \} ) + \Pr(\tau \leq T+1) \\
    &\leq \Pr( \exists t \geq 1 : \norm{ \bar{V}_t^{-1/2} Z_t } > \ell ) + \Pr(\tau \leq T+1) \\
    &\leq \Pr( \exists t \geq 1 : \norm{ \bar{V}_t^{-1/2} Z_t } > \ell ) + \sum_{s=2}^{T+1} \Pr( \norm{w_s} > R ) \\
    &\leq \Pr( \exists t \geq 1 : \norm{ \bar{V}_t^{-1/2} Z_t } > \ell ) + \bar{\delta} \:.
\end{align*}
Now set $\ell = 32 R^2 \left(d_2 \log{5} + \log\left( \frac{\det(\bar{V}_t)^{1/2} \det(V)^{-1/2}}{\delta} \right) \right)$
and using the fact that a $R$ bounded random variable is $2R$-sub-Gaussian, the claim now follows
by another application of Corollary 1 from \cite{abbasi11b}.
\end{proof}

With Lemma~\ref{lem:self_normalized} in place, we are ready to bound the 
martingale difference terms.
\begin{prop}
\label{prop:mds_normalized}
Suppose the hypothesis of Proposition~\ref{prop:singularbound} hold.
With probability at least $1-\delta$,
\begin{align*}
    \norm{(\Phi^\T \Phi)^{-1/2} \Phi^\T (\Xi - \Psi_+) q} &\leq (n+d) \sigma_w\sqrt{ \tau^2\rho^4\norm{\Sigma_0} + \norm{P_\infty} + \sigma_\eta^2 \norm{B}^2}(1+\norm{\Keval}^2) \norm{Q}_F \\
    &\qquad \times \polylog( n, \tau, \norm{\Sigma_0}, \norm{P_\infty}, \norm{\Kplay}, T/\delta, 1/\sigma_\eta ) \:, \\
    \norm{(\Phi^\T \Phi)^{-1/2} \Phi^\T (\Xi - \Psi_+) } &\leq (n+d)^2 \sigma_w\sqrt{ \tau^2\rho^4\norm{\Sigma_0} + \norm{P_\infty} + \sigma_\eta^2 \norm{B}^2}(1+\norm{\Keval}^2) \\
    &\qquad \times \polylog( n, \tau, \norm{\Sigma_0}, \norm{P_\infty}, \norm{\Kplay}, T/\delta, 1/\sigma_\eta ) \:.
\end{align*}
\end{prop}
\begin{proof}
For the proof, constants $c,c_i$ will denote universal constants.
Define two matrices:
\begin{align*}
    V_1 &:= c_1 \frac{\sigma_\eta^4}{(1+\norm{\Kplay}^2)^2} T \cdot I \:, \\
    V_2 &:= c_2 \frac{T}{\delta} (1+\norm{\Kplay}^2)^2 (\tau^2 \rho^2 n \norm{\Sigma_0} + \Tr(P_\infty))^2 \cdot I \:.
\end{align*}
By Proposition~\ref{prop:singularbound}, with probability at least $1-\delta/2$, we have that:
\begin{align*}
    V_1 \preceq \Phi^\T \Phi \preceq V_2 \:.
\end{align*}
Call this event $\calE_1$.

Next, we have:
\begin{align*}
    &\E[ x_{t+1}x_{t+1}^\T | x_t, u_t ] - x_{t+1}x_{t+1}^\T \\
    &= \E[ (Ax_t + Bu_t + w_t)(Ax_t + Bu_t + w_t)^\T | x_t,u_t] -  (Ax_t + Bu_t + w_t)(Ax_t + Bu_t + w_t)^\T \\
    &= (Ax_t + Bu_t)(Ax_t + Bu_t)^\T + \sigma_w^2 I \\
    &\qquad- (Ax_t + Bu_t)(Ax_t + Bu_t)^\T - (Ax_t+Bu_t)w_t^\T - w_t(Ax_t+Bu_t)^\T - w_tw_t^\T \\
    &= \sigma_w^2 I - w_tw_t^\T - (Ax_t+Bu_t)w_t^\T - w_t(Ax_t+Bu_t)^\T \:.
\end{align*}
Therefore,
\begin{align*}
    &\E[ \psi_{t+1} | x_t,u_t] - \psi_{t+1} \\
    &= \svec\left( \cvectwo{I}{\Keval} (\sigma_w^2 I - w_tw_t^\T - (Ax_t+Bu_t)w_t^\T - w_t(Ax_t+Bu_t)^\T ) \cvectwo{I}{\Keval}^\T  \right) \:.
\end{align*}
Taking the inner product of this term with $q$,
\begin{align*}
    &(\E[ \psi_{t+1} | x_t,u_t] - \psi_{t+1})^\T q \\
    &= \Tr\left((\sigma_w^2 I - w_tw_t^\T - (Ax_t+Bu_t)w_t^\T - w_t(Ax_t+Bu_t)^\T ) \cvectwo{I}{\Keval}^\T Q \cvectwo{I}{\Keval} \right) \\
    &= \Tr\left((\sigma_w^2 I - w_tw_t^\T)\cvectwo{I}{\Keval}^\T Q \cvectwo{I}{\Keval} \right) - 2 w_t^\T \cvectwo{I}{\Keval}^\T Q \cvectwo{I}{\Keval}(Ax_t+Bu_t) \:.
\end{align*}
By the Hanson-Wright inequality (see e.g.~\citet{rudelson13}), with probability at least $1-\delta/T$,
\begin{align*}
    \bigabs{\Tr\left((\sigma_w^2 I - w_tw_t^\T)\cvectwo{I}{\Keval}^\T Q \cvectwo{I}{\Keval} \right)} \leq c_1 \sigma_w^2 (1+\norm{\Keval}^2) \norm{Q}_F \log(T/\delta) \:.
\end{align*}
Now, let $\Lplay := A + B \Kplay$.
By Proposition 4.7 in \citet{tu18a}, with probability at least $1-\delta/T$,
\begin{align*}
    &\bigabs{w_t^\T \cvectwo{I}{\Keval}^\T Q \cvectwo{I}{\Keval}(Ax_t+Bu_t)} \\
    &\leq c_1 \sigma_w (1+\norm{\Keval}^2) \sqrt{ \norm{\Lplay^{t+1} \Sigma_0 (\Lplay^{t+1})^\T} +  \norm{\Lplay P_t \Lplay^\T} + \sigma_\eta^2 \norm{B}^2} \norm{Q}_F \log(T/\delta)  \\
    &\leq c_1 \sigma_w (1+\norm{\Keval}^2) \sqrt{ \tau^2 \rho^{2(t+1)} \norm{\Sigma_0} + \norm{P_\infty} + \sigma_\eta^2 \norm{B}^2} \norm{Q}_F \log(T/\delta) \:,
\end{align*}
where the inequality above comes from
$P_t \preceq P_\infty$ and $\Lplay P_\infty \Lplay^\T \preceq P_\infty$.
Therefore, we have:
\begin{align*}
    &\abs{(\E[ \psi_{t+1} | x_t,u_t] - \psi_{t+1})^\T v} \\
    &\leq c_2 (\sigma_w^2  +\sigma_w\sqrt{ \tau^2 \rho^{2(t+1)} \norm{\Sigma_0} + \norm{P_\infty} + \sigma_\eta^2 \norm{B}^2} ) (1+\norm{\Keval}^2) \norm{Q}_F \log(T/\delta) \\
    &\leq c_3 \sigma_w\sqrt{ \tau^2 \rho^{2(t+1)} \norm{\Sigma_0} + \norm{P_\infty} + \sigma_\eta^2 \norm{B}^2}(1+\norm{\Keval}^2) \norm{Q}_F \log(T/\delta) \:.
\end{align*}
The last inequality holds because $P_\infty \succeq \sigma_w^2 I$ and hence
$\sigma_w \leq \norm{P_\infty}^{1/2}$.
Therefore we can set 
\begin{align*}
R = c_3 \sigma_w\sqrt{ \tau^2 \rho^4 \norm{\Sigma_0} + \norm{P_\infty} + \sigma_\eta^2 \norm{B}^2}(1+\norm{\Keval}^2) \norm{Q}_F \log(T/\delta) \:,
\end{align*}
and invoke Lemma~\ref{lem:self_normalized} to conclude that with probability at least $1-\delta/2$,
\begin{align*}
    \norm{ (V_1 + \Phi^\T \Phi)^{-1/2} \Phi^\T (\Xi - \Psi_+) v } \leq c_4 (n+d) R + c_5 R \sqrt{ \log( \det((V_1 + \Phi^\T \Phi) V_1^{-1})^{1/2} / \delta) } \:.
\end{align*}
Call this event $\calE_2$.

For the remainder of the proof we work on $\calE_1 \cap \calE_2$, which has probability at least $1-\delta$.
Since $\Phi^\T \Phi \succeq V_1$, we have that $(\Phi^\T \Phi)^{-1} \leq 2 (V_1 + \Phi^\T \Phi)^{-1}$.
Therefore, by another application of Lemma~\ref{lem:self_normalized}:
\begin{align*}
    &\norm{ (\Phi^\T \Phi)^{-1/2} \Phi^\T (\Xi - \Psi_+)} \\
    &\leq \sqrt{2} \norm{ (V_1 + \Phi^\T \Phi)^{-1/2} \Phi^\T (\Xi - \Psi_+)} \\
    &\leq c_6 (n+d) R + c_7 R \sqrt{ \log( \det((V_1 + \Phi^\T \Phi) V_1^{-1})^{1/2} / \delta) } \\
    &\leq c_6 (n+d) R + c_7 R \sqrt{ \log( \det((V_1 + V_2) V_1^{-1})^{1/2} / \delta) } \\
    &\leq c_6 (n+d) R + c_8 R (n+d) \sqrt{ \log\left(  \frac{(1+\norm{\Kplay}^2)^4}{\delta} \frac{(\tau^2 \rho^2 n \norm{\Sigma_0} + \Tr(P_\infty))^2}{\sigma_\eta^4} \right)} \\
    &\leq c (n+d) R \polylog( n, \tau, \norm{\Sigma_0}, \norm{P_\infty}, \norm{\Kplay}, 1/\delta, 1/\sigma_\eta ) \:.
\end{align*}

Next, we bound:
\begin{align*}
    &\norm{ \E[\psi_{t+1}|x_t,u_t] - \psi_{t+1} } \\
    &\leq \bignorm{ \cvectwo{I}{\Keval} (\sigma_w^2 I - w_tw_t^\T) \cvectwo{I}{\Keval}^\T }_F +
          \bignorm{ \cvectwo{I}{\Keval} w_t (Ax_t + B u_t)^\T  \cvectwo{I}{\Keval}^\T }_F \\
          &\leq (1+\norm{\Keval}^2) (\norm{\sigma_w^2 I - w_tw_t^\T}_F + \norm{w_t (Ax_t + B u_t)^\T}_F) \:.
\end{align*}
Now, by standard Gaussian concentration results, with probability $1 - \delta/T$,
\begin{align*}
    \norm{\sigma_w^2 I - w_tw_t^\T}_F &\leq c\sigma_w^2( n + \log(T/\delta) ) \:,
\end{align*}
and also
\begin{align*}
    &\norm{w_t(Ax_t + B u_t)^\T}_F \\
    &\leq c\sigma_w(\sqrt{n}+\sqrt{\log(T/\delta)})  (\sqrt{\Tr(\Lplay^{t+1} \Sigma_0 (\Lplay^{t+1})^\T) + \Tr(\Lplay P_t \Lplay^\T) + \sigma_\eta^2 \norm{B}_F^2} \\
    &\qquad + \sqrt{\norm{\Lplay^{t+1} \Sigma_0 (\Lplay^{t+1})^\T} + \norm{\Lplay P_t \Lplay^\T} + \sigma_\eta^2 \norm{B}^2} \sqrt{\log(T/\delta)} ) \\
    &\leq c\sigma_w (n+d) \sqrt{ \tau^2 \rho^4 \norm{\Sigma_0} + \norm{P_\infty} + \sigma_\eta^2 \norm{B}} \log(T/\delta) \:.
\end{align*}
Therefore, with probability $1-\delta/T$,
\begin{align*}
    &\norm{ \E[\psi_{t+1}|x_t,u_t] - \psi_{t+1} } \\
    &\leq c (1+\norm{\Keval}^2) (n+d) \sigma_w\sqrt{\tau^2 \rho^4 \norm{\Sigma_0} + \norm{P_\infty} + \sigma_\eta^2 \norm{B}^2}\log(T/\delta) \:.
\end{align*}
\end{proof}

We are now in a position to prove Theorem~\ref{thm:lstd_q_estimation}.
We first observe that we can lower bound $\smin(I - L \otimes_s L)$
using the $(\tau,\rho)$-stability of $A+B\Keval$. This is because for $k \geq 1$,
\begin{align*}
    \norm{L^k} &= \bignorm{\cvectwo{I}{\Keval} (A+B\Keval)^{k-1} \rvectwo{A}{B}} \\
    &\leq 2\norm{\Keval}_+ \norm{\rvectwo{A}{B}} \tau \rho^{k-1} \\
    &\leq \frac{2\norm{\Keval}_+ \max\{1,\sqrt{\norm{A}^2 + \norm{B}^2}\} }{\rho} \tau \cdot \rho^k \:.
\end{align*}
Hence we see that $L$ is $(\frac{2\norm{\Keval}_+ \max\{1,\sqrt{\norm{A}^2 + \norm{B}^2}\}}{\rho} \tau, \rho)$-stable.
Next, we know that $\smin(I - L \otimes_s L) = \frac{1}{\norm{(I-L\otimes_s L)^{-1}}}$.
Therefore, for any unit norm $v$,
\begin{align*}
    \norm{(I-L \otimes_s L)^{-1} v} &= \norm{(I-L \otimes_s L)^{-1} \svec(\smat(v))} = \norm{\dlyap( L^\T, \smat(v) )}_F \\
    &\leq  \frac{4\norm{\Keval}_+^2 (\norm{A}^2+\norm{B}^2)_+ \tau^2}{\rho^2(1-\rho^2)} \:.
\end{align*}
Here, the last inequality uses Proposition~\ref{prop:dlyap_norm_bound}.
Hence we have the bound:
\begin{align*}
    \smin(I - L\otimes_s L) \geq \frac{\rho^2(1-\rho^2)}{4\norm{\Keval}_+^2 (\norm{A}^2+\norm{B}^2)_+ \tau^2} \:.
\end{align*}
By Proposition~\ref{prop:singularbound}, 
as long as $T \geq \Otilde(1)(n+d)^2$ with probability at least $1-\delta/2$:
\begin{align*}
    \smin(\Phi) \geq c \frac{\sigma_\eta^2}{\norm{\Kplay}_+^2} \sqrt{T} \:.
\end{align*}
By Proposition~\ref{prop:mds_normalized}, with probability at least $1-\delta/2$:
\begin{align*}
    \norm{(\Phi^\T \Phi)^{-1/2} \Phi^\T (\Xi - \Psi_+) q} &\leq (n+d) \sigma_w\sqrt{ \tau^2\rho^4\norm{\Sigma_0} + \norm{P_\infty} + \sigma_\eta^2 \norm{B}^2}\norm{\Keval}^2_+ \norm{Q^{\Keval}}_F \Otilde(1) \:, \\
    \norm{(\Phi^\T \Phi)^{-1/2} \Phi^\T (\Xi - \Psi_+) } &\leq (n+d)^2 \sigma_w\sqrt{\tau^2\rho^4\norm{\Sigma_0} + \norm{P_\infty} + \sigma_\eta^2 \norm{B}^2}\norm{\Keval}^2_+ \Otilde(1) \:.
\end{align*}
We first check the condition
\begin{align*}
    \frac{\norm{ (\Phi^\T \Phi)^{-1/2} \Phi^\T(\Xi-\Psi_+)}}{\smin(\Phi) \smin(I - L \otimes_s L)} \leq 1/2 \:,
\end{align*}
from Lemma~\ref{lem:basic_inequality}.
A sufficient condition is that $T$ satisfies:
\begin{align*}
    T &\geq \Otilde(1) \frac{\norm{\Kplay}_+^4}{\sigma_\eta^4} \cdot (n+d)^4 \sigma_w^2 (\tau^2\rho^4\norm{\Sigma_0} + \norm{P_\infty} + \sigma_\eta^2 \norm{B}^2) \\
    &\qquad \times\norm{\Keval}_+^4 \cdot \frac{\norm{\Keval}_+^4 (\norm{A}^2+\norm{B}^2)_+^2 \tau^4}{\rho^4(1-\rho^2)^2} \\
    &= \Otilde(1) \frac{\tau^4}{\rho^4(1-\rho^2)^2} \frac{(n+d)^4}{\sigma_\eta^4} \sigma_w^2 (\tau^2\rho^4\norm{\Sigma_0} + \norm{P_\infty} + \sigma_\eta^2 \norm{B}^2) \\
    &\qquad \times\norm{\Kplay}_+^4 \norm{\Keval}_+^8  (\norm{A}^4+\norm{B}^4)_+ \:.
\end{align*}
Once this condition on $T$ is satisfied, then we have that the error $\norm{\qh-q}$ is bounded by:
\begin{align*}
    &\Otilde(1) \frac{\norm{\Kplay}_+^2}{\sigma_\eta^2 \sqrt{T}} \cdot (n+d) \sigma_w\sqrt{ \tau^2\rho^4\norm{\Sigma_0} + \norm{P_\infty} + \sigma_\eta^2 \norm{B}^2} \\
    &\qquad \times\norm{\Keval}^2_+ \norm{Q^{\Keval}}_F \cdot \frac{\norm{\Keval}_+^2 (\norm{A}^2+\norm{B}^2)_+ \tau^2}{\rho^2(1-\rho^2)} \\
    &= \Otilde(1) \frac{\tau^2}{\rho^2(1-\rho^2)} \frac{(n+d)}{\sigma_\eta^2\sqrt{T}}\sigma_w\sqrt{ \tau^2\rho^4\norm{\Sigma_0} + \norm{P_\infty} + \sigma_\eta^2 \norm{B}^2} \\
    &\qquad \times\norm{\Kplay}_+^2 \norm{\Keval}_+^4 (\norm{A}^2+\norm{B}^2)_+\norm{Q^{\Keval}}_F \:.
\end{align*}
Theorem~\ref{thm:lstd_q_estimation} now follows from Lemma~\ref{lem:basic_inequality}.

\section{Analysis for LSPI}
\label{sec:lspi}

In this section we study the non-asymptotic behavior of LSPI.
Our analysis proceeds in two steps. We first understand the behavior of 
exact policy iteration on LQR. Then, we study the effects of introducing 
errors into the policy iteration updates.

\subsection{Exact Policy Iteration}
\label{sec:lspi:pi}

Exactly policy iteration works as follows. We start with a stabilizing controller
$K_0$ for $(A, B)$, and let $V_0$ denote its associated value function. 
We then apply the following recursions for $t=0, 1, 2, ...$:
\begin{align}
	K_{t+1} &= -(S + B^\T V_t B)^{-1} B^\T V_t A \:, \\
	V_{t+1} &= \dlyap(A + B K_{t+1}, S + K_{t+1}^\T R K_{t+1}) \label{eq:exact_policy_iteration} \:.
\end{align}
Note that this recurrence is related to, but different from, that of \emph{value iteration},
which starts from a PSD $V_0$ and recurses:
\begin{align*}
	V_{t+1} = A^\T V_t A - A^\T V_t B (S + B^\T V_t B)^{-1} B^\T V_t A + S \:.
\end{align*}
While the behavior of value iteration for LQR is well understood
(see e.g.\ \citet{lincoln06} or \citet{LinearEstimationBook}), the behavior of
policy iteration is less studied. \citet{fazel18} show that
policy iteration is equivalent to the Gauss-Newton method on the objective
$J(K)$ with a specific step-size, and give a simple analysis which shows
linear convergence to the optimal controller.
In this section, we present an analysis of the
behavior of exact policy iteration that builds on top of the fixed-point theory
from \citet{lee08}. A key component of our analysis is the following invariant metric
$\delta_\infty$ on positive definite matrices:
\begin{align*}
	\delta_\infty(A, B) := \norm{\log(A^{-1/2} B A^{-1/2})} \:.
\end{align*}
Various properties of $\delta_\infty$ are reviewed in Appendix~\ref{sec:app:metric}.

Our analysis proceeds as follows.
First, we note by the matrix inversion lemma:
\begin{align*}
    S + A^\T (B R^{-1} B^\T + V^{-1})^{-1} A = S + A^\T V A - A^\T V B(R + B^\T V B)^{-1} B^\T V A =: F(V) \:.
\end{align*}
Let $V_\star$ be the unique positive definite solution to $V = F(V)$.
For any positive definite $V$ we have by Lemma~\ref{lemma:contraction}:
\begin{align}
    \delta_\infty( F(V), V_\star ) \leq \frac{ \alpha }{\lambda_{\min}(S) + \alpha} \delta_\infty( V, V_\star ) \:, \label{eq:contraction}
\end{align}
with $\alpha = \max\{ \lambda_{\max}(A^\T V A), \lambda_{\max}(A^\T V_\star A) \}$.
Indeed, \eqref{eq:contraction} gives us another method to analyze value iteration, since
it shows that the Riccati operator $F(V)$ is contractive in the $\delta_\infty$ metric.
Our next result combines this contraction property
with the policy iteration analysis of \citet{bertsekas17}.
\begin{prop}[Policy Iteration for LQR]
\label{prop:PI_LQR}
Suppose that $S, R$ are positive definite and there exists a unique positive definite solution to
the discrete algebraic Riccati equation (DARE).
Let $K_0$ be a stabilizing policy for $(A, B)$ and let $V_0 = \dlyap(A + B K_0, S + K_0^\T R K_0)$.
Consider the following sequence of updates for $t= 0, 1, 2, ...$:
\begin{align*}
    K_{t+1} &= -(R + B^\T V_t B)^{-1} B^\T V_t A \:, \\
    V_{t+1} &= \dlyap(A+B K_{t+1}, S + K_{t+1}^\T R K_{t+1}) \:.
\end{align*}
The following statements hold:
\begin{enumerate}[(i)]
    \item $K_t$ stabilizes $(A, B)$ for all $t=0, 1, 2, ...$,
    \item $V_\star \preceq V_{t+1} \preceq V_t$ for all $t = 0, 1, 2, ...$,
    \item $\delta_{\infty}(V_{t+1}, V_\star) \leq \rho  \cdot \delta_{\infty}(V_t, V_\star)$ for all $t=0, 1, 2, ...$,
        with $\rho := \frac{\lambda_{\max}(A^\T V_0 A)}{\lambda_{\min}(S) + \lambda_{\max}(A^\T V_0 A)}$.
        Consequently, $\delta_{\infty}(V_t, V_\star) \leq \rho^t \cdot \delta_{\infty}(V_0, V_\star)$ for $t = 0, 1, 2, ...$.
\end{enumerate}
\end{prop}
\begin{proof}
We first prove (i) and (ii) using the argument of Proposition 1.3 from \citet{bertsekas17}.

Let $c(x, u) = x^\T S x + u^\T R u$, $f(x, u) = Ax + Bu$, and $V^K(x_1) = \sum_{t=1}^{\infty} c(x_t, u_t)$ with $x_{t+1} = f(x_t, u_t)$ and $u_t = K x_t$.
Let $V_t = V^{K_t}$.
With these definitions, we have that for all $x$:
\begin{align*}
    K_{t+1} x = \arg\min_{u} c(x, u) + V_t(f(x, u)) \:.
\end{align*}
Therefore,
\begin{align*}
    V_t(x) &= c(x, K_t x) + V_t(f(x, K_t x)) \\
    &\geq c(x, K_{t+1} x) + V_t(f(x, K_{t+1} x)) \\
    &= c(x, K_{t+1} x) + c(f(x, K_{t+1} x), K_{t}f(x, K_{t+1} x)) + V_t(f(f(x, K_{t+1} x), K_{t} f(x, K_{t+1} x))) \\
    &\geq c(x, K_{t+1} x) + c(f(x, K_{t+1} x), K_{t+1}f(x, K_{t+1} x)) + V_t(f(f(x, K_{t+1} x), K_{t+1} f(x, K_{t+1} x))) \\
    &\vdots \\
    &\geq V_{t+1}(x) \:.
\end{align*}
This proves (i) and (ii).

Now, observe that by partial minimization of a strongly convex quadratic:
\begin{align*}
    c(x, K_{t+1} x ) + V_t(f(x, K_{t+1} x)) &= \min_{u} c(x, u) + V_t(f(x, u)) \\
    &= x^\T ( S + A^\T V_t A - A^\T V_t B( R + B^\T V_t B)^{-1} B^\T V_t A ) x \\
    &= x^\T F(V_t) x \:.
\end{align*}
Combined with the above inequalities, this shows that $V_{t+1} \preceq F(V_t) \preceq V_t$.
Therefore, by \eqref{eq:contraction} and Proposition~\ref{prop:delta_ordering},
\begin{align*}
    \delta_{\infty}(V_{t+1}, V_\star) &\leq \delta_{\infty}(F(V_t), V_\star) \\
    &= \delta_{\infty}(F(V_t), F(V_\star)) \\
    &\leq \frac{ \alpha_t }{ \lambda_{\min}(Q) + \alpha_t } \delta_{\infty}(V_t, V_\star) \:,
\end{align*}
where $\alpha_t = \max\{ \lambda_{\max}(A^\T V_t A), \lambda_{\max}(A^\T V_\star A) \} = \lambda_{\max}(A^\T V_t A)$, since $V_\star \preceq V_t$.
But since $V_t \preceq V_0$, we can upper bound $\alpha_t \leq \lambda_{\max}(A^\T V_0 A)$.
This proves (iii).
\end{proof}

\subsection{Approximate Policy Iteration}
\label{sec:lspi:approx_PI}

We now turn to the analysis of approximate policy iteration.
Before analyzing Algorithm~\ref{alg:lspi_offline_v2}, we analyze a slightly more general
algorithm described in Algorithm~\ref{alg:approx_PI_offline}

\begin{center}
    \begin{algorithm}[htb]
    \caption{Approximate Policy Iteration for LQR (offline)}
    \begin{algorithmic}[1]
        \Require{Initial stabilizing controller $K_0$, $N$ number of policy iterations, $T$ length of rollout for estimation, $\sigma_\eta^2$ exploration variance.}
        \For{$t = 0, ..., N-1$}
            \State{Collect a trajectory $\calD_t = \{(x_k^{(t)}, u_k^{(t)}, x_{k+1}^{(t)})\}_{k=1}^{T}$ using input $u_k^{(t)} = K_0 x_k^{(t)} + \eta_k^{(t)}$, with $\eta_k^{(t)} \sim \calN(0, \sigma_\eta^2 I)$.}
            \State{$\Qh_{t} = \mathsf{EstimateQ}(\calD_t, K_t)$.}
            \State{$K_{t+1} = G(\Qh_{t})$.} [c.f.~\eqref{eq:T_def}]
        \EndFor
        \State{{\bf return} $K_N$.}
    \end{algorithmic}
    \label{alg:approx_PI_offline}
    \end{algorithm}
\end{center}

In Algorithm~\ref{alg:approx_PI_offline}, the procedure $\mathsf{EstimateQ}$
takes as input an off-policy trajectory $\calD_t$ and a policy $K_t$,
and returns an estimate $\Qhat_t$ of the true $Q$ function $Q_t$.
We will analyze Algorithm~\ref{alg:approx_PI_offline} first assuming that the procedure
$\mathsf{EstimateQ}$ delivers an estimate with a certain level of accuracy. 
In order to do this,
we define the sequence of variables:
\begin{enumerate}[(i)]
    \item $Q_{t}$ is true state-value function for $K_{t}$.
    \item $V_{t}$ is true value function for $K_{t}$.
    \item $\Kbar_{t+1} = G(Q_t)$.
    \item $\Vbar_{t}$ is true value function for $\Kbar_{t}$.
\end{enumerate}
The following proposition is our main result regarding Algorithm~\ref{alg:approx_PI_offline}.

\begin{prop}
\label{prop:api_lqr}
Consider the sequence of updates defined by Algorithm~\ref{alg:approx_PI_offline}.
Suppose we start with a stabilizing $K_0$ and let $V_0$ denote its value function.
Fix an $\varepsilon > 0$.
Define the following variables:
\begin{align*}
    \mu &:= \min\{\lambda_{\min}(S), \lambda_{\min}(R)\} \:, \\
    Q_{\max} &:= \max\{\norm{S},\norm{R}\} + 2(\norm{A}^2+\norm{B}^2)\norm{V_0} \:, \\
    \gamma &:= \frac{2\norm{A}^2\norm{V_0}}{\mu + 2\norm{A}^2\norm{V_0}} \:, \\
    N_0 &:= \ceil{ \frac{1}{1-\gamma} \log(2\delta_{\infty}(V_0, V_\star)/\varepsilon) } \:, \\
    \tau &:= \sqrt{\frac{2\norm{V_0}}{\mu}} \:, \\
    \rho &:= \sqrt{1-1/\tau^2} \:, \\
    \rhobar &:= \Avg(\rho, 1) \:.
\end{align*}
Let $N_1 \geq N_0$.
Suppose the estimates $\Qhat_t$ output by $\mathsf{EstimateQ}$ satisfy, for all $0 \leq t \leq N_1-1$,
$\Qhat_t \succeq \mu I$ and furthermore,
\begin{align*}
    \norm{\Qhat_t - Q_t} &\leq \min\left\{ \frac{\norm{V_0}}{N_1}, \varepsilon \mu(1-\gamma) \right\}  \left( \frac{\mu}{28} \frac{(1-\rhobar^2)^2}{\tau^5} \frac{1}{\norm{B}_+\max\{\norm{S},\norm{R}\}} \frac{\mu^3}{Q_{\max}^3} \right) \:.
\end{align*}
Then we have for any $N$ satisfying $N_0 \leq N \leq N_1$ the bound $\delta_\infty(V_N, V_\star) \leq \varepsilon$.
We also have that for all $0 \leq t \leq N_1$, $A+BK_t$ is $(\tau, \rhobar)$-stable
and $\norm{K_t} \leq 2Q_{\max}/\mu$.
\end{prop}
\begin{proof}
We first start by observing that if $V, V_0$ are value functions
satisfying $V \preceq V_0$, then their state-value functions also satisfy
$Q \preceq Q_0$.
This is because
\begin{align*}
    Q &= \bmattwo{S}{0}{0}{R} + \cvectwo{A^\T}{B^\T} V \rvectwo{A}{B} \\
      &\preceq \bmattwo{S}{0}{0}{R} + \cvectwo{A^\T}{B^\T} V_0 \rvectwo{A}{B} = Q_0 \:.
\end{align*}
From this we also see that any state-value function
satisfies $Q \succeq \bmattwo{S}{0}{0}{R}$.

The proof proceeds as follows.
We observe that since $\Vbar_{t+1} \preceq V_t$ (Proposition~\ref{prop:PI_LQR}-(ii)):
\begin{align*}
    V_t &= V_t - \Vbar_{t} + \Vbar_{t} - V_{t-1} + V_{t-1} \preceq V_t - \Vbar_{t} + V_{t-1} \:.
\end{align*}
Therefore, by triangle inequality we have $\norm{V_t} \leq \norm{V_t - \Vbar_{t}} + \norm{V_{t-1}}$.
Supposing for now that we can ensure for all $1 \leq t \leq N_1$:
\begin{align}
    \norm{V_t - \Vbar_{t}} \leq \frac{\norm{V_0}}{N} \:, \label{eq:constraint_one}
\end{align}
unrolling the recursion for $\norm{V_t}$ for $N_1$ steps ensures that $\norm{V_t} \leq 2\norm{V_0}$ for all $0 \leq t \leq N_1$.
Furthermore,
\begin{align*}
    \norm{Q_t} &\leq \max\{ \norm{S}, \norm{R} \} + \norm{\rvectwo{A}{B}}^2 \norm{V_t} \\
    &\leq \max\{ \norm{S}, \norm{R} \} + 2(\norm{A}^2+\norm{B}^2) \norm{V_0} \\
    &= Q_{\max} \:.
\end{align*}
for all $0 \leq t \leq N_1$.

Now, by triangle inequality and Proposition~\ref{prop:PI_LQR}-(iii), for all $0 \leq t \leq N_1-1$,
\begin{align}
    \delta_\infty(V_{t+1}, V_\star) &\leq \delta_\infty(V_{t+1}, \Vbar_{t+1}) + \delta_\infty(\Vbar_{t+1}, V_\star) \nonumber \\
    &\leq \delta_\infty(V_{t+1}, \Vbar_{t+1}) + \gamma \cdot \delta_\infty(V_t, V_\star) \nonumber \\
    &\leq \frac{\norm{V_{t+1} - \Vbar_{t+1}}}{\mu} + \gamma \cdot \delta_\infty(V_t, V_\star) \label{eq:one_step_delta_inequality} \:,
\end{align}
where $\gamma = \frac{ 2\norm{A}^2 \norm{V_0}}{\mu + 2\norm{A}^2\norm{V_0}}$,
and the last inequality uses Proposition~\ref{prop:delta_ub} combined with the fact that $V_{t+1} \succeq \mu I$ and
$\Vbar_{t+1} \succeq \mu I$.

We now focus on bounding $\norm{V_{t+1} - \Vbar_{t+1}}$. To do this,
we first bound $\norm{K_{t+1} - \Kbar_{t+1}}$, and then use the Lyapunov perturbation result
from Section~\ref{sec:app:perturbation}.
First, observe the simple bounds:
\begin{align*}
    \norm{\Kbar_{t+1}} &= \norm{G(Q_t)} \leq \frac{\norm{Q_t}}{\mu} \leq \frac{Q_{\max}}{\mu} \:, \\
    \norm{K_{t+1}} &= \norm{G(\Qh_t)} \leq \frac{\norm{\Qh_t}}{\mu} \leq \frac{\Delta + Q_{\max}}{\mu} \leq \frac{2Q_{\max}}{\mu} \:.
\end{align*}
where the second bound uses the assumption that the estimates $\Qhat_t$ satisfy $\Qhat_t \succeq \mu I$
and $\norm{\Qhat_t - Q_t} \leq \Delta$ with
\begin{align}
    \Delta \leq Q_{\max} \:. \label{eq:constraint_basic}
\end{align}
Now, by Proposition~\ref{prop:control_bound} we have:
\begin{align*}
    \norm{K_{t+1} - \Kbar_{t+1}} &= \norm{G(\Qhat_t) - G(Q_t)} \\
    &\leq \frac{(1+\norm{\Kbar_{t+1}}) \norm{\Qhat_t-Q_t}}{\mu} \\
    &\leq \frac{(1+Q_{\max}/\mu) \Delta}{\mu} \\
    &\leq \frac{2 Q_{\max}}{\mu^2} \Delta \:.
\end{align*}
Above, the last inequality holds since $Q_{\max} \geq \mu$ by definition.

By Proposition~\ref{prop:ordered_stability}, because $\Vbar_{t+1} \preceq V_t$, we know that
$\Kbar_{t+1}$ satisfies for all $k \geq 0$:
\begin{align*}
    \norm{(A+B \Kbar_{t+1})^k} &\leq \sqrt{\frac{\norm{V_t}}{\lambda_{\min}(S)}} \cdot \sqrt{1 - \lambda_{\min}(V_t^{-1} S)}^k \\
    &\leq \sqrt{\frac{2\norm{V_0}}{\mu}} \sqrt{1- \frac{\mu}{2\norm{V_0}}}^k = \tau \cdot \rho^k \:.
\end{align*}
Let us now assume that $\Delta$ satisfies:
\begin{align}
    \frac{2Q_{\max}}{\mu^2} \cdot \Delta \leq \frac{1-\rho}{2 \tau \norm{B}} \:. \label{eq:constraint_two}
\end{align}
Then by Lemma~\ref{lem:robust_control},
we know that $\norm{(A+BK_{t+1})^k} \leq \tau \cdot \rhobar^k$.
Hence, we have that $A+BK_{t+1}$ is $(\tau, \rhobar)$-stable.

Next, by the Lyapunov perturbation result of Proposition~\ref{prop:dlyap_perturbation},
\begin{align*}
    &\norm{V_{t+1} - \Vbar_{t+1}} \\
    &= \norm{\dlyap(A+BK_{t+1}, S + K_{t+1}^\T R K_{t+1}) - \dlyap(A+B\Kbar_{t+1}, S + \Kbar_{t+1}^\T R \Kbar_{t+1})} \\
    &\leq \frac{\tau^2}{1-\rhobar^2} \norm{ K_{t+1}^\T R K_{t+1} - \Kbar_{t+1}^\T R \Kbar_{t+1} } \\
    &\qquad+ \frac{\tau^4}{(1-\rhobar^2)^2} \norm{ B (K_{t+1} - \Kbar_{t+1}) } ( \norm{A+BK_{t+1}} + \norm{A+B\Kbar_{t+1}}) \norm{S + \Kbar_{t+1}^\T R \Kbar_{t+1}} \:.
\end{align*}
We bound:
\begin{align*}
    \norm{  K_{t+1}^\T R K_{t+1} - \Kbar_{t+1}^\T R \Kbar_{t+1}  } &\leq \norm{R} \norm{K_{t+1} - \Kbar_{t+1}}(\norm{K_{t+1}} + \norm{\Kbar_{t+1}}) \\
    &\leq \frac{6\norm{R} Q_{\max}^2}{\mu^3} \Delta \:, \\
    \norm{B (K_{t+1}-\Kbar_{t+1})} &\leq \frac{2\norm{B} Q_{\max}}{\mu^2} \Delta \:, \\
    \max\{\norm{A+B K_{t+1}},\norm{A+B \Kbar_{t+1}}\} &\leq \tau \:, \\
    \norm{S + \Kbar_{t+1}^\T R \Kbar_{t+1}} &\leq \norm{S} + \frac{\norm{R}Q_{\max}^2}{\mu^2} \:.
\end{align*}
Therefore,
\begin{align*}
    \norm{V_{t+1} - \Vbar_{t+1}} &\leq \frac{\tau^2}{1-\rhobar^2}\frac{6\norm{R} Q_{\max}^2}{\mu^3} \Delta + 8\frac{\tau^5}{(1-\rhobar^2)^2} \norm{B}\max\{\norm{S},\norm{R}\} \frac{Q_{\max}^3}{\mu^4} \Delta \\
    &= \frac{1}{\mu}\left( \frac{\tau^2}{1-\rhobar^2}\frac{6\norm{R} Q_{\max}^2}{\mu^2} + 8\frac{\tau^5}{(1-\rhobar^2)^2} \norm{B}\max\{\norm{S},\norm{R}\} \frac{Q_{\max}^3}{\mu^3} \right) \Delta \\
    &\leq \frac{14}{\mu} \frac{\tau^5}{(1-\rhobar^2)^2} \norm{B}_+ \max\{\norm{S},\norm{R}\} \frac{Q_{\max}^3}{\mu^3} \Delta \:.
\end{align*}
Now suppose that $\Delta$ satisfies:
\begin{align}
    \Delta &\leq  \frac{1}{2}\varepsilon\mu(1-\gamma) \left( \frac{\mu}{14} \frac{(1-\rhobar^2)^2}{\tau^5} \frac{1}{ \norm{B}_+ \max\{\norm{S},\norm{R}\}} \frac{\mu^3}{Q_{\max}^3} \right) \nonumber \\
    &= \frac{\varepsilon}{28} \mu^2(1-\gamma) \frac{(1-\rhobar^2)^2}{\tau^5}\frac{1}{\norm{B}_+ \max\{\norm{S},\norm{R}\}} \frac{\mu^3}{Q_{\max}^3} \:, \label{eq:constraint_three}
\end{align}
we have for all $t \leq N_1-1$ from \eqref{eq:one_step_delta_inequality}:
\begin{align*}
    \delta_{\infty}(V_{t+1}, V_\star) &\leq (1-\gamma)\varepsilon/2 + \gamma \cdot \delta_{\infty}(V_t, V_\star) \:.
\end{align*}
Unrolling this recursion, we have that for any $N \leq N_1$:
\begin{align*}
    \delta_{\infty}(V_N, V_\star) &\leq \gamma^N \cdot \delta_{\infty}(V_0, V_\star) + \varepsilon/2 \:.
\end{align*}
Now observe that for any  $N \geq N_0 := \ceil{ \frac{1}{1-\gamma} \log(2\delta_{\infty}(V_0, V_\star)/\varepsilon) }$,
we obtain:
\begin{align*}
    \delta_{\infty}(V_N, V_\star) \leq \varepsilon \:.
\end{align*}
The claim now follows by combining our four requirements on $\Delta$
given in \eqref{eq:constraint_basic}, \eqref{eq:constraint_one}, \eqref{eq:constraint_two}, and \eqref{eq:constraint_three}.
\end{proof}

We now proceed to make several simplifications to Proposition~\ref{prop:api_lqr}
in order to make the result more presentable. These simplifications come with the
tradeoff of introducing extra conservatism into the bounds.

Our first simplification of Proposition~\ref{prop:api_lqr} is the following
corollary.
\begin{cor}
\label{cor:api_lqr}
Consider the sequence of updates defined by Algorithm~\ref{alg:approx_PI_offline}.
Suppose we start with a stabilizing $K_0$ and let $V_0$ denote its value function.
Define the following variables:
\begin{align*}
    \mu &:= \min\{\lambda_{\min}(S), \lambda_{\min}(R)\} \:, \\
    L &:= \max\{\norm{S},\norm{R}\} + 2(\norm{A}^2+\norm{B}^2 + 1)\norm{V_0}_+ \:, \\
    N_0 &:= \ceil{ (1+L/\mu) \log(2\delta_{\infty}(V_0, V_\star)/\varepsilon) } \:.
\end{align*}
Fix an $N_1 \geq N_0$ and suppose that
\begin{align}
    \varepsilon \leq \frac{1}{\mu}\left(1 + \frac{L}{\mu}\right) \frac{\norm{V_0}}{N_1} \:. \label{eq:eps_requirement}
\end{align}
Suppose the estimates $\Qhat_t$ output by $\mathsf{EstimateQ}$ satisfy, for all $0 \leq t \leq N_1 - 1$,
$\Qhat_t \succeq \mu I$ and furthermore,
\begin{align*}
    \norm{\Qhat_t - Q_t} \leq \frac{\varepsilon}{448} \frac{\mu}{\mu+L} \left(\frac{\mu}{L}\right)^{19/2} \:.
\end{align*}
Then we have for any $N_0 \leq N \leq N_1$ that $\delta_\infty(V_N, V_\star) \leq \varepsilon$.
We also have that for any $0 \leq t \leq N_1$,
that $A+BK_t$ is $(\sqrt{L/\mu}, \Avg(\sqrt{1-\mu/L},1))$-stable and $\norm{K_t} \leq 2L/\mu$.
\end{cor}
\begin{proof}
First, observe that the map $x \mapsto \frac{x}{\mu+x}$ is increasing, and therefore
$\gamma \leq \frac{L}{\mu+L}$
which implies that $1-\gamma \geq \frac{\mu}{\mu+L}$.
Therefore if $\varepsilon \leq \frac{1}{\mu}\left(1+\frac{L}{\mu}\right) \frac{\norm{V_0}}{N_1}$ holds,
then we can bound:
\begin{align*}
    \min\left\{ \frac{\norm{V_0}}{N_1}, \varepsilon \mu(1-\gamma) \right\} \geq \varepsilon \mu \left(\frac{\mu}{\mu+L}\right) \:.
\end{align*}
Next, observe that
\begin{align*}
    1- \rhobar^2 &= (1+\rhobar)(1-\rhobar) 
    = (1 + 1/2 + \rho/2)(1/2 - \rho/2) 
    \geq (1+\rho)(1-\rho)/4 
    = (1-\rho^2)/4 \:.
\end{align*}
Therefore,
\begin{align*}
    (1-\rhobar^2)^2 \geq (1 - (1-\mu/L))^2/16 =(1/16) (\mu/L)^2 \:.
\end{align*}
We also have that $\tau \leq \sqrt{\frac{L}{\mu}}$.
This means we can bound:
\begin{align*}
    \frac{\mu}{28} \frac{(1-\rhobar^2)^2}{\tau^5} \frac{1}{\norm{B}_+ \max\{\norm{S},\norm{R}\}} \frac{\mu^3}{Q_{\max}^3}  \geq
    \frac{\mu}{28\cdot 16} (\mu/L)^{5/2 + 2} \frac{\mu^3}{L^5} = \frac{1}{448 L} \left(\frac{\mu}{L}\right)^{17/2} \:.
\end{align*}
Therefore,
\begin{align*}
    \min\left\{ \frac{\norm{V_0}}{N_1}, \varepsilon \mu(1-\gamma) \right\} \frac{\mu}{28} \frac{(1-\rhobar^2)^2}{\tau^5} \frac{1}{\norm{B}_+ \max\{\norm{S},\norm{R}\}} \frac{\mu^3}{Q_{\max}^3}
    \geq \frac{\varepsilon}{448} \left(\frac{\mu}{\mu+L}\right) \left(\frac{\mu}{L}\right)^{19/2} \:.
\end{align*}
The claim now follows from Proposition~\ref{prop:api_lqr}.
\end{proof} 

Corollary~\ref{cor:api_lqr} gives a guarantee in terms of $\delta_\infty(V_N, V_\star) \leq \varepsilon$.
By Proposition~\ref{prop:opnorm_to_delta}, this implies
a bound on the error of the value functions $\norm{V_N - V_\star} \leq \calO(\varepsilon)$
for $\varepsilon \leq 1$. In the next corollary, we show we can also control the error
$\norm{K_N - K_\star} \leq \calO(\varepsilon)$.

\begin{cor}
\label{cor:api_lqr_K}
Consider the sequence of updates defined by Algorithm~\ref{alg:approx_PI_offline}.
Suppose we start with a stabilizing $K_0$ and let $V_0$ denote its value function.
Define the following variables:
\begin{align*}
    \mu &:= \min\{\lambda_{\min}(S), \lambda_{\min}(R)\} \:, \\
    L &:= \max\{\norm{S},\norm{R}\} + 2(\norm{A}^2+\norm{B}^2 + 1)\norm{V_0}_+ \:, \\
    N_0 &:= \bigceil{ (1+L/\mu) \log\left(\frac{2\log(\norm{V_0}/\lambda_{\min}(V_\star))}{\varepsilon}\right) } \:.
\end{align*}
Suppose that $\varepsilon > 0$ satisfies:
\begin{align*}
    \varepsilon \leq \min\left\{ 1, \frac{2\log(\norm{V_0}/\lambda_{\min}(V_\star))}{e}, \frac{\norm{V_\star}^2}{8\mu^2\log(\norm{V_0}/\lambda_{\min}(V_\star))} \right\} \:.
\end{align*}
Suppose we run Algorithm~\ref{alg:approx_PI_offline} for $N := N_0+1$ iterations.
Suppose the estimates $\Qhat_t$ output by $\mathsf{EstimateQ}$ satisfy, for all $0 \leq t \leq N_0$,
$\Qhat_t \succeq \mu I$ and furthermore,
\begin{align}
    \norm{\Qhat_t - Q_t} \leq \frac{\varepsilon}{448} \frac{\mu}{\mu+L} \left(\frac{\mu}{L}\right)^{19/2} \:. \label{eq:Qhat_condition}
\end{align}
We have that:
\begin{align*}
    \norm{K_{N} - K_\star} \leq 5 \left(\frac{L}{\mu}\right)^2 \varepsilon
\end{align*}
and that $A+BK_t$ is $(\sqrt{L/\mu}, \Avg(\sqrt{1-\mu/L}, 1))$-stable and $\norm{K_t} \leq 2L/\mu$
for all $0 \leq t \leq N$.
\end{cor}
\begin{proof}
We set $N_1 = N_0 + 1$.
From this, we compute:
\begin{align*}
    \norm{K_{N_1} - K_\star} &= \norm{G(\Qhat_{N_0}) - G(Q_\star)} \\
    &\stackrel{(a)}{\leq} \frac{(1+\norm{G(Q_\star)})}{\mu} \norm{\Qhat_{N_0} - Q_\star} \\
    &\leq \frac{(1+\norm{G(Q_\star)})}{\mu} (\norm{\Qhat_{N_0} - Q_{N_0}} + \norm{Q_{N_0} - Q_\star}) \\
    &= \frac{(1+\norm{G(Q_\star)})}{\mu} \left(\norm{\Qhat_{N_0} - Q_{N_0}} + \bignorm{ \cvectwo{A^\T}{B^\T} (V_{N_0} - V_\star) \rvectwo{A}{B}}\right) \\
    &\leq \frac{(1+\norm{G(Q_\star)})}{\mu}(\norm{\Qhat_{N_0} - Q_{N_0}}+\norm{\rvectwo{A}{B}}^2\norm{V_{N_0}-V_\star}) \\
    &\stackrel{(b)}{\leq} \frac{(1+\norm{G(Q_\star)})}{\mu}\left( \frac{\varepsilon}{448} \frac{\mu}{\mu+L} \left(\frac{\mu}{L}\right)^{19/2} + \norm{\rvectwo{A}{B}}^2\norm{V_{N_0}-V_\star}\right) \\
    &\stackrel{(c)}{\leq} \frac{(1+\norm{G(Q_\star)})}{\mu}\left( \frac{\varepsilon}{448} \frac{\mu}{\mu+L} \left(\frac{\mu}{L}\right)^{19/2} + e(\norm{A}^2+\norm{B}^2) \norm{V_\star} \varepsilon \right) \\
    &\leq \frac{2L}{\mu^2}\left( \frac{1}{448} \frac{\mu}{\mu+L} \left(\frac{\mu}{L}\right)^{19/2} + 2L \right)  \varepsilon \\
    &= \left( \frac{1}{224} \frac{1}{\mu+L} \left(\frac{\mu}{L}\right)^{17/2} + 4 \left(\frac{L}{\mu}\right)^2 \right) \varepsilon \\
    &\leq 5 \left(\frac{L}{\mu}\right)^2 \varepsilon \:.
\end{align*}
Above,
(a) follows from Proposition~\ref{prop:control_bound},
(b) follows from the bound on $\norm{\Qhat_{N_0}-Q_{N_0}}$ from Corollary~\ref{cor:api_lqr}, and
(c) follows from Proposition~\ref{prop:opnorm_to_delta} and the fact
that $\delta_\infty(V_{N_0},V_\star)\leq\varepsilon$ from Corollary~\ref{cor:api_lqr}.

Next, we observe that since $V_0 \succeq V_\star$:
\begin{align*}
    \delta_\infty(V_0, V_\star) = \log(\norm{V_\star^{-1/2} V_0 V_\star^{-1/2}}) \leq \log( \norm{V_0}/\lambda_{\min}(V_\star)) \:.
\end{align*}
Hence we can upper bound $N_0$ from Corollary~\ref{cor:api_lqr} by:
\begin{align*}
    N_0 &= 2 (1+L/\mu) \log(2\log(\norm{P_0}/\lambda_{\min}(V_\star))/\varepsilon) \:.
\end{align*}
From \eqref{eq:eps_requirement}, the requirement on $\varepsilon$ is that:
\begin{align*}
    \varepsilon \leq \min\left\{ \frac{\norm{V_0}}{2\mu} \frac{1}{\log\left(\frac{2\log(\norm{V_0}/\lambda_{\min}(V_\star))}{\varepsilon}\right)}, 1 \right\} \:.
\end{align*}
We will show with Proposition~\ref{prop:eps_simplification} that a sufficient condition is that:
\begin{align*}
    \varepsilon \leq \min\left\{ 1, \frac{2\log(\norm{V_0}/\lambda_{\min}(V_\star))}{e}, \frac{\norm{V_\star}^2}{8\mu^2\log(\norm{V_0}/\lambda_{\min}(V_\star))} \right\} \:.
\end{align*}
\end{proof}

With Corollary~\ref{cor:api_lqr_K} in place, we are now ready to 
prove Theorem~\ref{thm:lspi_estimation}.

\begin{proof}[Proof of Theorem~\ref{thm:lspi_estimation}]
Let $L_0 := A + B K_0$ and let $(\tau,\rho)$ be such that $L_0$ is $(\tau,\rho)$-stable.
We know we can pick $\tau=\sqrt{L/\mu}$ and $\rho = \sqrt{1-\mu/L}$.
The covariance $\Sigma_t$ of $x_t$ satisfies:
\begin{align*}
    \Sigma_t = L_0^t \Sigma_0 (L_0^t)^\T + P_t \preceq \tau^2 \rho^{2t} \norm{\Sigma_0} I + P_\infty \:.
\end{align*}
Hence for either $t=0$ or $t \geq \log(\tau)/(1-\rho)$, 
$\norm{\Sigma_t} \leq \norm{\Sigma_0} + \norm{P_\infty}$.
Therefore, if the trajectory length $T \geq \log(\tau)/(1-\rho)$, then 
the operator norm of the initial covariance for every invocation of LSTD-Q can be bounded by
$\norm{\Sigma_0} + \norm{P_\infty}$,
and therefore the proxy variance \eqref{eq:proxy_variance} can be bounded by:
\begin{align*}
    \sigmabar^2 &\leq \tau^2\rho^4 \norm{\Sigma_0} + (1+\tau^2\rho^4)\norm{P_\infty} + \sigma_\eta^2\norm{B}^2 \\
    &\leq 2(L/\mu) (\norm{\Sigma_0} + \norm{P_\infty} + \sigma_\eta^2\norm{B}^2) \:.
\end{align*}

By Corollary~\ref{cor:api_lqr_K}, 
when condition \eqref{eq:Qhat_condition} holds, we have that
$A+BK_t$ is $(\tau, \Avg(\rho, 1))$ stable,
$\norm{K_t} \leq 2L/\mu$, and $\norm{Q_t} \leq L$ for all $0 \leq t \leq N_0+1$.
We now define $\varepsilonbar := 5 (L/\mu)^2 \varepsilon$.
If we can ensure that 
\begin{align}
    \norm{\Qhat_t - Q_t} \leq \frac{1}{2240} \left(\frac{\mu}{\mu+L}\right) \left(\frac{\mu}{L}\right)^{23/2} \varepsilonbar \:, \label{eq:qhat_cond_two}
\end{align} 
then if
\begin{align*}
    \varepsilonbar \leq 5 \left(\frac{L}{\mu}\right)^2 \min\left\{ 1, \frac{2\log(\norm{V_0}/\lambda_{\min}(V_\star))}{e}, \frac{\norm{V_\star}^2}{8\mu^2\log(\norm{V_0}/\lambda_{\min}(V_\star))} \right\} \:,
\end{align*}
then by Corollary~\ref{cor:api_lqr_K} we ensure that $\norm{K_N - K} \leq \varepsilonbar$.
By Theorem~\ref{thm:lstd_q_estimation}, \eqref{eq:qhat_cond_two} can be ensured 
by first observing that $Q_t \succeq \mu I$ and therefore
for any symmetric $\Qh$ we have:
\begin{align*}
    \norm{\mathsf{Proj}_\mu(\Qh) - Q_t} \leq \norm{\mathsf{Proj}_\mu(\Qh) - Q_t}_F \leq \norm{\Qh - Q_t}_F \:.
\end{align*}
Above, the last inequality holds because
$\mathsf{Proj}_\mu(\cdot)$ is the Euclidean projection operator associated with $\norm{\cdot}_F$
onto the convex set $\{ Q : Q \succeq \mu I \:, \:\: Q = Q^\T \}$. 
Now combining \eqref{eq:lstdq_T_bound} and \eqref{eq:lstdq_T_requirement}
and using the bound $\frac{\tau^2}{\rho^2(1-\rho^2)} \leq \frac{(L/\mu)^2}{1-\mu/L}$:
\begin{align*}
    T \geq \Otilde(1) \max\bigg\{ &(n+d)^2, \\
    &\frac{L^2}{(1-\mu/L)^2} \left(\frac{L}{\mu}\right)^{17} \frac{(n+d)^4}{\sigma_\eta^4} \sigma_w^2 (\norm{\Sigma_0} + \norm{P_\infty} + \sigma_\eta^2 \norm{B}^2) , \\
    &\frac{1}{\varepsilonbar^2} \frac{L^4}{(1-\mu/L)^2}\left( \frac{L}{\mu} \right)^{42} \frac{(n+d)^3}{\sigma_\eta^4}   \sigma_w^2 ( \norm{\Sigma_0} + \norm{P_\infty} + \sigma_\eta^2\norm{B}^2  ) \bigg\} \:.
\end{align*}
Theorem~\ref{thm:lspi_estimation} now follows.
\end{proof}

\section{Analysis for Adaptive LSPI}
\label{sec:adaptive}

In this section we develop our analysis for Algorithm~\ref{alg:lspi_online}.
We start by presenting a meta adaptive algorithm (Algorithm~\ref{alg:meta_online}) and lay out sufficient
conditions for the meta algorithm to achieve sub-linear regret. We then
specialize the meta algorithm to use LSPI as a sub-routine.

\begin{center}
    \begin{algorithm}[htb]
    \caption{General Adaptive LQR Algorithm}
    \begin{algorithmic}[1]
        \Require{Initial stabilizing controller $K^{(0)}$, number of epochs $E$, epoch multiplier $\Tmult$.}
        \For{$i = 0, ..., E-1$}
            \State{Set $T_i = \Tmult 2^i$.}
            \State{Set $\sigma_{\eta,i}^2 = \sigma_w^2\left(\frac{1}{2^i}\right)^{1/(1+\alpha)}$.}
            \State{Roll system forward $T_i$ steps with input $u_t^{(i)} = K^{(i)} x_t^{(i)} + \eta_t^{(i)}$, where
            $\eta_t^{(i)} \sim \calN(0, \sigma_{\eta,i}^2 I)$.}
            \State{Let $\calD_i = \{ (x_t^{(i)}, u_t^{(i)}, x_{t+1}^{(i)}) \}_{t=0}^{T_i}$.}
            \State{Set $K^{(i+1)} = \mathsf{EstimateK}( K^{(i)}, \calD_i )$.}
        \EndFor
    \end{algorithmic}
    \label{alg:meta_online}
    \end{algorithm}
\end{center}

Algorithm~\ref{alg:meta_online} is the general form of the $\varepsilon$-greedy
strategy for adaptive LQR recently described in \citet{dean18} and \citet{mania19}.
We study Algorithm~\ref{alg:meta_online} under the following assumption 
regarding the sub-routine $\mathsf{EstimateK}$.

\begin{assume}
\label{assumption:estimate_K}
We assume there exists two functions
$\Creq, \Cerr$ and $\alpha \geq 1$ such that the following holds.
Suppose the controller $K^{(i)}$ that generates $\calD_i$
stabilizes $(A,B)$ and $V^{(i)}$ is its associated value function,
the initial condition $x_0^{(i)} \sim \calN(0, \Sigma_0^{(i)})$, 
and that the trajectory $\calD_i$ is collected via
$u_t^{(i)} = K^{(i)} x_t^{(i)} + \eta_t^{(i)}$ with $\eta_t^{(i)} \sim \calN(0, \sigma_{\eta,i}^2 I)$.
For any $0 < \varepsilon < \Creq(\norm{V^{(i)}})$ and any $\delta \in (0, 1)$, 
as long as $\abs{\calD_i}$ satisfies:
\begin{align}
	\abs{\calD_i} \geq \frac{\Cerr(\norm{V^{(i)}}, \norm{\Sigma_0^{(i)}})}{\varepsilon^2} \frac{1}{\sigma_{\eta,i}^{2\alpha}} \polylog(\abs{\calD_i}, 1/\sigma_{\eta,i}^\alpha, 1/\delta, 1/\varepsilon) \:,
\end{align}
then we have with probability at least $1-\delta$ that $\norm{K^{(i+1)} - K_\star} \leq \varepsilon$.
We also assume the function $\Creq$ (resp. $\Cerr$) is monotonically decreasing (resp. increasing)
with respect to its arguments, and that the functions are allowed to depend
in any way on the problem parameters $(A,B,S,R,n,d,\sigma_w^2, K_\star, P_\star)$
\end{assume}

Before turning to the analysis of Algorithm~\ref{alg:meta_online}, we state a 
simple proposition that bounds the covariance matrix along the trajectory
induced by Algorithm~\ref{alg:meta_online}.
\begin{prop}
\label{prop:covariance_bound}
Fix a $j \geq 1$.
Let $\Sigma_0^{(j)}$ denote the covariance matrix of $x_0^{(j)}$.
Suppose that for all $0 \leq i < j$ each $K^{(i)}$ stabilizes $(A,B)$ 
$A+BK^{(i)}$ is $(\tau,\rho)$-stable.
Also suppose that $\sigma_{\eta,i} \leq \sigma_w$ and that
\begin{align*}
    \Tmult \geq \frac{1}{2(1-\rho)} \log\left(\frac{n\tau^2}{\rho^2}\right) \:.
\end{align*}
We have that:
\begin{align*}
    \Tr(\Sigma_0^{(j)}) \leq \sigma_w^2(1+\norm{B}^2) n \frac{\tau^2}{(1-\rho^2)^2} \:.
\end{align*}
\end{prop}
\begin{proof}
Let $L_i = A + B K^{(i)}$.
We write:
\begin{align*}
    \E[ \norm{x_0^{(i)}}^2 ] &= \E[\E[ \norm{x_0^{(i)}}^2 | x_0^{(i-1)} ]] \\
    &= \E[\E[ \Tr(x_0^{i}(x_0^{i})^\T) | x_0^{(i-1)} ]] \\
    &\leq \E[ \Tr( L_{i-1}^{T_{i-1}} x_0^{(i-1)} (x_0^{(i-1)})^\T (L_{i-1}^{T_{i-1}})^\T ) ] + (\sigma_w^2 + \sigma_{\eta,i-1}^2 \norm{B}^2) n \frac{\tau^2}{1-\rho^2} \\
    &\leq n \tau^2 \rho^{2 T_{i-1}} \E[\norm{x_0^{(i-1)}}^2] + \sigma_w^2(1+\norm{B}^2) n \frac{\tau^2}{1-\rho^2} \:.
\end{align*}
We have that $x_0^{(0)} = 0$.
Hence if we choose $\Tmult$ such that $n \tau^2 \rho^{2\Tmult} \leq \rho^2$,
we obtain the recurrence:
\begin{align*}
    \E[ \norm{x_0^{(i)}}^2 ] \leq \rho^2 \E[\norm{x_0^{(i-1)}}^2] + \sigma_w^2(1+\norm{B}^2) n \frac{\tau^2}{1-\rho^2} \:,
\end{align*}
and therefore $\E[ \norm{x_0^{(i)}}^2 ] \leq \sigma_w^2(1+\norm{B}^2) n \frac{\tau^2}{(1-\rho^2)^2}$ for all $i$.
This is ensured if
\begin{align*}
    \Tmult \geq \frac{1}{2(1-\rho)} \log(n\tau^2/\rho^2) \:.
\end{align*}
\end{proof}

Next, we state a lemma that relates the instantaneous cost
to the expected cost. The proof is based on the Hanson-Wright inequality, and
appears in \citet{dean18}.
Let the notation $J(K;\Sigma)$ denote the 
infinite horizon average LQR cost when the feedback $u_t = K x_t$ is played and
when the process noise is $w_t \sim \calN(0, \Sigma)$.
Explicitly:
\begin{align}
	J(K;\Sigma) = \Tr(\Sigma V(K)) \:, \:\:
	V(K) = \dlyap(A+BK, S + K^\T R K) \:. \label{eq:extra_J_notation}
\end{align}
With this notation, we have the following lemma.
\begin{lem}[Lemma D.2, \cite{dean18}]
\label{lem:instant_cost_to_expectation}
Let $x_0 \sim \calN(0, \Sigma_0)$ and suppose that $u_t = K x_t + \eta_t$ with $A+BK$ as $(\tau,\rho)$-stable and
$\eta_t \sim \calN(0, \sigma_\eta^2 I)$.
We have that with probability at least $1-\delta$:
\begin{align*}
    \sum_{t=1}^{T} x_t^\T Q x_t + u_t^\T R u_t &\leq T J(K; \sigma_w^2 I + \sigma_\eta^2 BB^\T) \\
    &\qquad+ c \sqrt{nT} \frac{\tau^2}{(1-\rho)^2} (\norm{\Sigma_0}+ \sigma_w^2 + \sigma_\eta^2 \norm{B}^2) \norm{Q+K^\T R K} \log(1/\delta) \:.
\end{align*}
\end{lem}

Finally, we state a second order perturbation result from \citet{fazel18},
which was recently used by \citet{mania19} to study certainty equivalent controllers.
\begin{lem}[Lemma 12, \cite{fazel18}]
\label{lem:fast_rate}
Let $K$ stabilize $(A, B)$ with $A+BK$ as $(\tau,\rho)$-stable, and let $K_\star$ be the optimal LQR controller for $(A,B,Q,R)$ and $V_\star$ be the optimal value function.
We have that:
\begin{align*}
    J(K) - J_\star \leq \sigma_w^2 \frac{\tau^2}{1-\rho^2} \norm{R + B^\T V_\star B} \norm{K-K_\star}_F^2 \:.
\end{align*}
\end{lem}

With these tools in place, we are ready to state our main result regarding the
regret incurred (c.f. \eqref{eq:regret_def}) by Algorithm~\ref{alg:meta_online}.

\begin{prop}
\label{prop:regret_proof}
Fix a $\delta \in (0, 1)$.
Suppose that $\mathsf{EstimateK}$ satisfies Assumption~\ref{assumption:estimate_K}.
Let the initial feedback $K^{(0)}$ stabilize $(A,B)$ and let $V^{(0)}$ denote its associated
value function.
Also let $K_\star$ denote the optimal LQR controller and let $V_\star$ denote the optimal value function.
Let $\Gamma_\star = 1 + \max\{\norm{A},\norm{B},\norm{V^{(0)}}, \norm{V_\star}, \norm{K^{(0)}}, \norm{K_\star},\norm{Q},\norm{R}\}$.
Define the following bounds:
\begin{align*}
	\Kmax &:= \Gamma_\star \:, \\
	\Vmax &:= 4\frac{\Gamma_\star^5}{\lambda_{\min}(S)^2} \:, \\
	\Sigma_{\max} &:= 4 \sigma_w^2 n \frac{\Gamma_\star^4}{\lambda_{\min}(S)^2} \:.
\end{align*}
Suppose that $\Tmult$ satisfies:
\begin{align*}
	\Tmult \geq \max\left\{1,  \frac{\Gamma_\star^8}{\lambda_{\min}(S)^4}, \frac{1}{\Creq^4(\Vmax)}\right\}\frac{\Cerr^2(\Vmax, \Sigma_{\max})}{\sigma_w^4} \mathrm{poly}(\alpha) \polylog(1/\sigma_w, E/\delta) \:.
\end{align*}
With probability at least $1-\delta$, we have that:
\begin{align*}
	\mathsf{Regret}(T) &\leq \sigma_w^{2(1-\alpha)} d\frac{\Gamma_\star^7}{\lambda_{\min}(S)^2} \Cerr^2(\Vmax, \Sigma_{\max}) \left(\frac{T+1}{\Tmult}\right)^{\alpha/(\alpha+1)} \polylog(T/\delta) \\
	&\qquad + \Tmult \Gamma_\star^2 J_\star \left(\frac{T+1}{\Tmult}\right)^{\alpha/(\alpha+1)} \\
	&\qquad + \calO(1) n^{3/2} \sqrt{T} \sigma_w^2 \frac{\Gamma_\star^9}{\lambda_{\min}(S)^4}  \log(T/\delta) + o_T(1) \:.
\end{align*}
\end{prop}
\begin{proof}
We state the proof assuming that $T$ is at an epoch boundary for simplicity.
Each epoch has length $T_i = \Tmult 2^i$.
Let $T_0 + T_1 + ... + T_{E-1} = T$. This means that $E = \log_2((T+1)/\Tmult)$.

We start by observing that by Proposition~\ref{prop:ordered_stability},
we have that $A+BK^{(0)}$ is $(\tau,\rho)$-stable
for $\tau :=  \sqrt{\norm{V^{(0)}}/\lambda_{\min}(S)}$
and $\rho := \sqrt{1 - \lambda_{\min}(S)/\norm{V^{(0)}}}$.
We will show that $A+BK^{(i)}$ is $(\tau, \rhobar)$-stable for $i=1, ..., E-1$
for $\rhobar := \Avg(\rho, 1)$.
By Lemma~\ref{lem:robust_control}, this occurs if we can ensure that
$\norm{K^{(i)} - K_\star} \leq \frac{(1-\rho)}{2\tau \norm{B}}$. for $i=1, ..., E-1$.

We will also construct bounds $\Kmax, \Vmax, \Sigma_{\max}$ such that
$\norm{K^{(i)}} \leq \Kmax$, $\norm{V^{(i)}} \leq \Vmax$,
and $\norm{\Sigma^{(i)}} \leq \Sigma_{\max}$ for all $0 \leq i \leq E-1$.
We set the bounds as:
\begin{align*}
	\Kmax &:= \max\{\norm{K^{(0)}}, \norm{K_\star} + 1\} \:, \\
	\Vmax &:= \max\{\norm{V^{(0)}}, \frac{\tau^2}{1-\rhobar^2}(\norm{Q}+\norm{R}\Kmax^2) \} \:, \\
	\Sigma_{\max} &:= \sigma_w^2 (1+\norm{B}^2)n \frac{\tau^2}{1-\rhobar^2} \:.
\end{align*} 
In what follows, we will use the shorthand
$\Creq = \Creq(\Vmax)$ and
$\Cerr = \Cerr(\Vmax, \Sigma_{\max})$.

Before we continue, we first argue that our choice of $\Tmult$ satisfies for all $i=1, ..., E-1$:
\begin{align}
	T_{i-1} \geq \max\{1, \frac{\tau^2\norm{B}^2}{4(1-\rhobar)^2}, \frac{1}{\Creq^2}\} \frac{\Cerr}{\sigma_{\eta,i}^{2\alpha}} \polylog(T_{i-1}, 1/\sigma_{i,\eta}^\alpha, 1/\sigma_w, E/\delta)\:. \label{eq:master_inequality}
\end{align}
Rearranging, this is equivalent to:
\begin{align*}
	\Tmult\geq \max\{1, \frac{\tau^2\norm{B}^2}{4(1-\rhobar)^2}, \frac{1}{\Creq^2}\} 2 \Cerr \sigma_w^{-2} \frac{1}{(2^i)^{1/(1+\alpha)}} \polylog(\Tmult 2^i, (2^i)^{\alpha/(1+\alpha)}, 1/\sigma_w, E/\delta)\:.
\end{align*}
We first remove the dependence on $i$ on the RHS by taking
the maximum over all $i$.
By Proposition~\ref{prop:T_helper}, it suffices to take
$\Tmult$ satisfying:
\begin{align*}
	\Tmult \geq \max\{1, \frac{\tau^2\norm{B}^2}{4(1-\rhobar)^2}, \frac{1}{\Creq^2}\}\frac{\Cerr}{\sigma_w^2} \mathrm{poly}(\alpha) \polylog(\Tmult, 1/\sigma_w, E/\delta) \:.
\end{align*}
We now remove the implicit dependence on $\Tmult$.
By Proposition~\ref{prop:lambert_w}, it suffices to take
$\Tmult$ satisfying:
\begin{align*}
	\Tmult &\geq \max\{1, \frac{\tau^2\norm{B}^2}{4(1-\rhobar)^2}, \frac{1}{\Creq^2}\}\frac{\Cerr}{\sigma_w^2} \\
    &\qquad \times \mathrm{poly}(\alpha) \polylog(1/\sigma_w, E/\delta, \tau, \norm{B}, 1/(1-\rhobar), 1/\Creq, \Cerr) \:.
\end{align*}
We are now ready to proceed.

First we look at the base case $i=0$. Clearly, 
the bounds work for $i=0$ by definition.
Now we look at epoch $i \geq 1$ and we assume the bounds hold for $\ell=0, ..., i-1$.
For $i \geq 1$ we define $\varepsilon_i$ as:
\begin{align*}
	\varepsilon_i := \inf\left\{ \varepsilon \in (0, 1) : T_{i-1} \geq \frac{\Cerr}{\varepsilon^2} \frac{1}{\sigma_{\eta,i}^{2\alpha}} \polylog(T_{i-1}, 1/\sigma_{\eta,i}^\alpha, E/\delta, 1/\varepsilon) \right\} \:.
\end{align*}
By Proposition~\ref{prop:recover_eps}, we have that
as long as 
\begin{align}
	T_{i-1} \geq \Cerr \frac{1}{\sigma_{\eta,i}^{2\alpha}} \polylog(T_{i-1}, 1/\sigma_{\eta,i}^{\alpha}, E/\delta) \:, \label{eq:cond_one}
\end{align}
then we have that $\varepsilon_i$ satisfies:
\begin{align}
	\varepsilon_i^2 \leq \frac{ \Cerr}{T_{i-1} \sigma_{\eta,i}^{2\alpha}} \polylog( T_{i-1}, 1/\sigma_{\eta,i}^\alpha, E/\delta ) \:. \label{eq:eps_inequality}
\end{align}
But \eqref{eq:cond_one} is implied by \eqref{eq:master_inequality}, so
we know that \eqref{eq:eps_inequality} holds.
Therefore, we have $\norm{K^{(i)} - K_\star} \leq \varepsilon_i$.

Now by \eqref{eq:eps_inequality}, if:
\begin{align}
\frac{ \Cerr}{T_{i-1} \sigma_{\eta,i}^{2\alpha}} \polylog( T_{i-1}, 1/\sigma_{\eta,i}^\alpha, E/\delta ) \leq \min\{ 1, \frac{(1-\rhobar)^2}{4 \tau^2 \norm{B}^2}, \Creq^2 \} \:, \label{eq:cond_two}
\end{align}
then the following is true:
\begin{align*}
	\varepsilon_i \leq \min\{ 1, (1-\rhobar)/(2\tau \norm{B}), \Creq \} \:.
\end{align*}
However, \eqref{eq:cond_two} is also implied by \eqref{eq:master_inequality},
so we have by Assumption~\ref{assumption:estimate_K}:
\begin{align*}
	\norm{K^{(i)} - K_\star} \leq \min\{ 1, (1-\rhobar)/(2\tau \norm{B}) \} \:.
\end{align*}
This has several implications. First, it implies that:
\begin{align*}
	\norm{K^{(i)}} \leq \norm{K_\star} + 1 \leq \Kmax \:.
\end{align*}
Next, it implies by Lemma~\ref{lem:robust_control} that
$A+BK^{(i)}$ is $(\tau,\rhobar)$-stable.
Next, by Proposition~\ref{prop:covariance_bound}, it implies that
$\norm{\Sigma^{(i)}} \leq \Sigma_{\max}$.
Finally, letting $L_i := A + B K^{(i)}$, we have that:
\begin{align*}
	\norm{V^{(i)}} &= \bignorm{ \sum_{\ell=0}^{\infty} (L_i)^\ell (Q + (K^{(i)})^\T R K^{(i)}) (L_i^\T)^\ell  } \\
	&\leq \frac{\tau^2}{1-\rhobar^2} (\norm{Q} + \norm{R} \Kmax^2) \\
	&\leq \Vmax \:.
\end{align*}
Thus, by induction we have that
$\norm{K^{(i)}} \leq \Kmax$, $\norm{V^{(i)}} \leq \Vmax$,
and $\norm{\Sigma^{(i)}} \leq \Sigma_{\max}$ for all $0 \leq i \leq E-1$.

We are now ready to bound the regret.
From \eqref{eq:extra_J_notation}, we see the relation $J(K; \sigma_w^2 I + \sigma_\eta^2 BB^\T) \leq \left(1 + \frac{\sigma_\eta^2 \norm{B}^2}{\sigma_w^2}\right) J(K; \sigma_w^2 I)$ holds.
Therefore by Lemma~\ref{lem:instant_cost_to_expectation} and Lemma~\ref{lem:fast_rate},
\begin{align*}
    \sum_{t=1}^{T} x_t^\T Q x_t + u_t^\T R u_t - T J_\star &\leq T \left(1 + \frac{\sigma_\eta^2 \norm{B}^2}{\sigma_w^2}\right)  (J_\star + \sigma_w^2 \frac{\tau^2}{1-\rhobar^2} \norm{R + B^\T P_\star B} \norm{K - K_\star}_F^2 ) - T J_\star \\
    &\qquad+ c \sqrt{nT} \frac{\tau^2}{(1-\rhobar)^2} (\norm{P_0}+ \sigma_w^2 + \sigma_\eta^2 \norm{B}^2) \norm{Q+K^\T R K} \log(1/\delta) \\
    &\leq T (\sigma_w^2 + \sigma_\eta^2 \norm{B}^2) \frac{\tau^2}{1-\rhobar^2} (\norm{R}+\norm{P_\star}\norm{B}^2) \norm{K - K_\star}_F^2 + T \frac{\sigma_\eta^2 \norm{B}^2}{\sigma_w^2} J_\star \\
    &\qquad+ c \sqrt{nT} \frac{\tau^2}{(1-\rhobar)^2} (\norm{P_0}+ \sigma_w^2 + \sigma_\eta^2 \norm{B}^2) (\norm{Q}+\norm{K}^2 \norm{R} ) \log(1/\delta) \:.
\end{align*}
Using the inequality above,
\begin{align*}
    \mathsf{Regret}(T) &= \sum_{i=0}^{E-1} \sum_{t=1}^{T_i} (x_t^{(i)})^\T Q (x_t^{(i)}) + (u_t^{(i)})^\T R (u_t^{(i)}) - T J_\star \\
    &\leq \sum_{i=0}^{E-1} T_i \sigma_w^2 (1+\norm{B}^2) \frac{\tau^2}{1-\rhobar^2} (\norm{R}+\norm{V_\star}\norm{B}^2) \norm{K^{(i)} - K_\star}_F^2 + T_i \frac{\sigma_{\eta,i}^2 \norm{B}^2}{\sigma_w^2} J_\star \\
    &\qquad+ c \sqrt{nT_i} \sigma_w^2 (1+\norm{B}^2) n \frac{\tau^4}{(1-\rhobar^2)^4} (\norm{Q}+\Kmax^2 \norm{R} ) \log(E/\delta) \\
    &\leq \calO(1) + \sum_{i=1}^{E-1} \sigma_w^2 (1+\norm{B}^2) \frac{d\tau^2}{1-\rhobar^2} (\norm{R}+\norm{V_\star}\norm{B}^2) C_{\mathrm{err}}^2 \frac{2}{\sigma_{\eta,i}^{2\alpha}} \polylog(E/\delta, 1/\sigma_{\eta,i})  \\
    &\qquad+ T_i \frac{\sigma_{\eta,i}^2 \norm{B}^2}{\sigma_w^2} J_\star \\
    &\qquad+ c \sqrt{nT_i} \sigma_w^2 (1+\norm{B}^2) n \frac{\tau^4}{(1-\rhobar^2)^4} (\norm{Q}+\Kmax^2 \norm{R} ) \log(E/\delta) \\
    &= 2 \sum_{i=1}^{E-1} \sigma_w^{2-2\alpha} (1+\norm{B}^2) \frac{d\tau^2}{1-\rhobar^2} (\norm{R}+\norm{V_\star}\norm{B}^2) C_{\mathrm{err}} (2^i)^{\alpha/(1+\alpha)}  \polylog(E/\delta, 1/\sigma_{\eta,i}) \\
    &\qquad+ \Tmult  (2^i)^{\alpha/(1+\alpha)} \norm{B}^2 J_\star \\
    &\qquad+ c \sqrt{nT_i} \sigma_w^2 (1+\norm{B}^2) n \frac{\tau^4}{(1-\rhobar^2)^4} (\norm{Q}+\Kmax^2 \norm{R} ) \log(E/\delta) + \calO(1) \\
    &\leq \sigma_w^{2-2\alpha} (1+\norm{B}^2) \frac{d\tau^2}{1-\rhobar^2} (\norm{R}+\norm{V_\star}\norm{B}^2) C_{\mathrm{err}}^2\frac{\alpha+1}{\alpha} \left(\frac{T+1}{\Tmult}\right)^{\alpha/(\alpha+1)} \polylog(T/\delta) \\
    &\qquad+ \Tmult\norm{B}^2 J_\star \frac{\alpha+1}{\alpha} \left(\frac{T+1}{\Tmult}\right)^{\alpha/(\alpha+1)} \\
    &\qquad+ \calO(1)\sqrt{nT} \sigma_w^2 (1+\norm{B}^2) n \frac{\tau^4}{(1-\rhobar^2)^4} (\norm{Q}+\Kmax^2 \norm{R} ) \log(T/\delta) + \calO(1) \:.
\end{align*}
The last inequality holds because:
\begin{align*}
    \sum_{i=1}^{E-1} (2^i)^{\alpha/(1+\alpha)} &\leq \int_1^E (2^x)^{\alpha/(1+\alpha)} \: dx \leq \frac{1}{\log{2}}\frac{\alpha+1}{\alpha} (2^E)^{\alpha/(\alpha+1)}
    = \frac{1}{\log{2}} \frac{\alpha+1}{\alpha} \left(\frac{T+1}{\Tmult}\right)^{\alpha/(\alpha+1)} \:.
\end{align*}
Now observe that we can bound
\begin{align*}
	\Kmax &\leq \Gamma_\star \:, \\
	\Vmax &\leq 4\frac{\tau^2}{1-\rho^2} \Gamma_\star^3 \:, \\
	\Sigma_{\max} &\leq 4 \sigma_w^2 n \Gamma_\star^2 \frac{\tau^2}{1-\rho^2} \:, \\
	\frac{\tau^2}{1-\rho^2} &\leq \frac{\Gamma_\star^2}{\lambda_{\min}(S)^2}\:. 
\end{align*}
Therefore:
\begin{align*}
	\mathsf{Regret}(T) &\leq \sigma_w^{2(1-\alpha)} d\frac{\Gamma_\star^7}{\lambda_{\min}(S)^2} \Cerr^2 \left(\frac{T+1}{\Tmult}\right)^{\alpha/(\alpha+1)} \polylog(T/\delta) \\
	&\qquad + \Tmult \Gamma_\star^2 J_\star \left(\frac{T+1}{\Tmult}\right)^{\alpha/(\alpha+1)} \\
	&\qquad + \calO(1) n^{3/2} \sqrt{T} \sigma_w^2 \frac{\Gamma_\star^9}{\lambda_{\min}(S)^4}  \log(T/\delta) + o_T(1) \:.
\end{align*}
\end{proof}

We now turn to the proof of Theorem~\ref{thm:lspi_regret} and analyze Algorithm~\ref{alg:lspi_online}
by applying Proposition~\ref{prop:regret_proof} 
with LSPI (Section~\ref{sec:lspi}) taking the place of $\mathsf{EstimateK}$.
To apply Proposition~\ref{prop:regret_proof}, we use Theorem~\ref{thm:lspi_estimation}
to compute the bounds $\Creq,\Cerr$ that are needed for Assumption~\ref{assumption:estimate_K}
to hold.
The following proposition will be used to work out these bounds.
\begin{prop}
\label{prop:dlyap_coupling}
Let $P_1 = \dlyap(L, M_1)$
and $P_2 = \dlyap(L^\T, M_2)$,
and suppose both $M_1$ and $M_2$ are $n \times n$ positive definite.
We have that:
\begin{align*}
	\norm{P_1} \leq n \frac{\norm{M_1}}{\smin(M_2)} \norm{P_2} \:.
\end{align*}
\end{prop}
\begin{proof}
We start with the observation that $\Tr(M_2 P_1) = \Tr(M_1 P_2)$.
Then we lower bound $\Tr(M_2 P_1) \geq \smin(M_2) \Tr(P_1) \geq \smin(M_2) \norm{P_1}$,
and upper bound $\Tr(M_1 P_2) \leq \norm{M_1} \Tr(P_2) \leq n \norm{M_1} \norm{P_2}$.
\end{proof}

We use Proposition~\ref{prop:dlyap_coupling} to compute the
following upper bound for $P_\infty$:
\begin{align*}
	\norm{P_\infty} \leq n \frac{\sigma_w^2 + \sigma_\eta^2 \norm{B}^2}{ \lambda_{\min}(S)} \norm{V_\star}
	\leq \sigma_w^2 n \frac{\Gamma_\star^2}{\lambda_{\min}(S)} \:.
\end{align*}

We first compute the $\Creq$ term from \eqref{eq:lspi_Creq}:
\begin{align*}
	\Creq(\norm{V^{(i)}}) = \min\left\{ 1, \frac{2\log(\norm{V^{(i)}}/\lambda_{\min}(V_\star))}{e}, \frac{\norm{V_\star}^2}{8\mu^2\log(\norm{V^{(i)}}/\lambda_{\min}(V_\star))} \right\} \:.
\end{align*}
We see that $\Creq$ is monotonically decreasing in $\norm{V^{(i)}}$.

Next we compute $\Cerr$ from \eqref{eq:lspi_Cerr}.
First we see that $\alpha = 2$.
Observing we can upper bound $L \leq \Gamma_\star^2 \norm{V^{(i)}}_+$,
we have that:
\begin{align*}
	\Cerr(\norm{V^{(i)}}, \norm{\Sigma_0^{(i)}}) = \frac{\Gamma_\star^{94}}{(1-\mu/(\Gamma_\star^2 \norm{V_\star}_+))^2} \frac{\norm{V^{(i)}}_+^{47}}{\mu^{43}} (n+d)^4 \sigma_w^2 \left( \norm{\Sigma_0^{(i)}}  + \sigma_w^2 n \frac{\Gamma_\star^2}{\lambda_{\min}(S)} + \sigma_w^2 \Gamma_\star^2 \right) \:.
\end{align*}
We see that $\Cerr$ is monotonically increasing in both $\norm{V^{(i)}}$ and $\norm{\Sigma_0^{(i)}}$.
This gives the proof of Theorem~\ref{thm:lspi_regret}.


\section{Properties of the Invariant Metric}
\label{sec:app:metric}

Here we review relevant properties of the invariant metric $\delta_\infty(A, B) = \norm{\log(A^{-1/2} B A^{-1/2})}$ over positive definite matrices.

\begin{lem}[c.f.~\cite{lee08}]
\label{lemma:properties}
Suppose that $A$ is positive semidefinite and $X, Y$ are positive definite.
Also suppose that $M$ is invertible.
We have:
\begin{enumerate}[(i)]
    \item $\delta_\infty(X, Y) = \delta_\infty(X^{-1}, Y^{-1}) = \delta_\infty(M X M^\T, M Y M^\T)$.
    \item $\delta_\infty(A+X, A+Y) \leq \frac{\alpha}{\alpha+\beta} \delta_\infty(X, Y)$, where
        $\alpha = \max\{\lambda_{\max}(X), \lambda_{\max}(Y)\}$ and $\beta = \lambda_{\min}(A)$.
\end{enumerate}
\end{lem}

\begin{lem}[c.f.~Theorem 4.4, \cite{lee08}]
\label{lemma:contraction}
Consider the map $f(X) = A + M(B + X^{-1})^{-1} M^\T$, where $A, B$ are PSD and $X$ is positive definite.
Suppose that $X, Y$ are two positive definite matrices and $A$ is invertible.
We have:
\begin{align*}
    \delta_\infty(f(X), f(Y)) \leq \frac{\max\{ \lambda_1(M X M^\T), \lambda_1(M Y M^\T) \} }{\lambda_{\min}(A) + \max\{ \lambda_1(M X M^\T), \lambda_1(M Y M^\T) \}  } \delta_\infty(X, Y) \:.
\end{align*}
\end{lem}
\begin{proof}
We first assume that $M$ is invertible.
Using the properties of $\delta_\infty$ from Lemma~\ref{lemma:properties},
we have:
\begin{align*}
    \delta_\infty(f(X), f(Y)) &= \delta_\infty(A + M(B+X^{-1})^{-1}M^\T, A + M(B+Y^{-1})^{-1}M^\T) \\
    &\leq \frac{\alpha}{\lambda_{\min}(A) + \alpha} \delta_\infty(M(B+X^{-1})^{-1}M^\T, M(B+Y^{-1})^{-1}M^\T) \\
    &= \frac{\alpha}{\lambda_{\min}(A) + \alpha} \delta_\infty( (B+X^{-1})^{-1}, (B+Y^{-1})^{-1} ) \\
    &= \frac{\alpha}{\lambda_{\min}(A) + \alpha} \delta_\infty( B+X^{-1}, B+Y^{-1}) \\
    &\leq \frac{\alpha}{\lambda_{\min}(A) + \alpha} \delta_\infty(X^{-1}, Y^{-1}) \\
    &= \frac{\alpha}{\lambda_{\min}(A) + \alpha} \delta_\infty(X, Y) \:,
\end{align*}
where $\alpha = \max\{ \lambda_{\max}( M(B+X^{-1})^{-1}M^\T ), \lambda_{\max}( M(B+X^{-1})^{-1}M^\T ) \}$.
Now, we observe that:
\begin{align*}
    B + X^{-1} \succeq X^{-1} \Longleftrightarrow (B+X^{-1})^{-1} \preceq X \:.
\end{align*}
This means that $M (B + X^{-1})^{-1} M^\T \preceq M X M^\T$ and similarly
$M(B + Y^{-1})^{-1} M^\T \preceq M Y M^\T$.
This proves the claim when $M$ is invertible.
When $M$ is not invertible, use a standard limiting argument.
\end{proof}

\begin{prop}
\label{prop:delta_ub}
Suppose that $A,B$ are positive definite matrices satisfying $A \succeq \mu I$, $B \succeq \mu I$.
We have that:
\begin{align*}
    \delta_\infty(A,B) \leq \frac{\norm{A-B}}{\mu} \:.
\end{align*}
\end{prop}
\begin{proof}
We have that:
\begin{align*}
    \norm{ A^{-1/2} B A^{-1/2} } = \norm{ A^{-1/2}(B-A)A^{-1/2} + I } \leq 1 + \frac{\norm{B-A}}{\mu} \:.
\end{align*}
Taking log on both sides and using $\log(1+x) \leq x$ for $x \geq 0$ yields the claim.
\end{proof}

\begin{prop}
\label{prop:delta_ordering}
Suppose that $B \preceq A_1 \preceq A_2$ are all positive definite matrices.
We have that:
\begin{align*}
    \delta_\infty(A_1, B) \leq \delta_\infty(A_2, B) \:.
\end{align*}
\end{prop}
\begin{proof}
The chain of orderings implies that:
\begin{align*}
    I \preceq B^{-1/2} A_1 B^{-1/2} \preceq B^{-1/2} A_2 B^{-1/2} \:.
\end{align*}
Therefore:
\begin{align*}
    \delta_\infty(A_1, B) = \log{\lambda_{\max}( B^{-1/2} A_1 B^{-1/2})} \leq \log{\lambda_{\max}( B^{-1/2} A_2 B^{-1/2})} = \delta_\infty(A_2, B) \:.
\end{align*}
Each step requires careful justification.
The first equality holds because
$I \preceq B^{-1/2} A_1 B^{-1/2}$ and the second inequality uses the monotonicity of the scalar function $x \mapsto \log{x}$ on $\R_+$
in addition to $B^{-1/2} A_1 B^{-1/2} \preceq B^{-1/2} A_2 B^{-1/2}$.
\end{proof}

\begin{prop}
\label{prop:opnorm_to_delta}
Suppose that $A,B$ are positive definite matrices with $B \succeq A$.
We have that:
\begin{align*}
    \norm{A-B} \leq \norm{A} (\exp(\delta_\infty(A,B))-1) \:.
\end{align*}
Furthermore, if $\delta_\infty(A,B) \leq 1$ we have:
\begin{align*}
    \norm{A-B} \leq e \norm{A} \delta_\infty(A,B) \:.
\end{align*}
\end{prop}
\begin{proof}
The assumption that $B \succeq A$ implies that $A^{-1/2} B A^{-1/2} \succeq I$
and that $\norm{A-B} = \lambda_{\max}(B-A)$.
Now observe that:
\begin{align*}
    \norm{A-B} &= \lambda_{\max}(B-A) \\
    &= \lambda_{\max}(A^{1/2} (A^{-1/2} B A^{-1/2} - I) A^{1/2}) \\
    &\leq \norm{A} \lambda_{\max}(A^{-1/2} B A^{-1/2} - I) \\
    &= \norm{A} (\lambda_{\max}(A^{-1/2} B A^{-1/2}) - 1) \\
    &= \norm{A} (\exp(\delta_{\infty}(A,B)) - 1) \:.
\end{align*}
This yields the first claim.
The second follows from the crude bound that $e^x \leq 1 + e x$ for $x \in (0, 1)$.
\end{proof}

\section{Useful Perturbation Results}
\label{sec:app:perturbation}

Here we collect various perturbation results which are used
in Section~\ref{sec:lspi:approx_PI}.

\begin{lem}[Lemma B.1, \cite{tu18b}]
\label{lem:robust_control}
Suppose that $K_0$ stabilizes $(A, B)$, and satisfies
$\norm{(A+B K_0)^k} \leq \tau \rho^k$ for all $k$ with $\tau \geq 1$ and $\rho \in (0, 1)$.
Suppose that $K$ is a feedback matrix that satisfies $\norm{K - K_0} \leq \frac{1-\rho}{2\tau\norm{B}}$.
Then we have that $K$ stabilizes $(A, B)$ and satisfies
$\norm{(A+BK)^k} \leq \tau \Avg(\rho, 1)^k$.
\end{lem}

\begin{lem}[Lemma 1, \cite{mania19}]
\label{lem:strongly_convex_min_perturb}
  Let $f_1, f_2$ be two $\mu$-strongly convex twice differentiable
  functions.  Let $x_1 = \arg\min_x f_1(x)$ and
  $x_2 = \arg\min_x f_2(x)$.  Suppose  $\norm{ \nabla f_1(x_2)} \leq \varepsilon$, then
  $  \norm{x_1 - x_2} \leq \frac{\varepsilon}{\mu}$.
\end{lem}

\begin{prop}
\label{prop:control_bound}
Let $M \succeq \mu I$ and $N \succeq \mu I$ be a positive definite matrices
partitioned as $M = \bmattwo{M_{11}}{M_{12}}{M_{12}^\T}{M_{22}}$ and similarly for $N$.
Let $T(M) = -M_{22}^{-1} M_{12}^\T$. We have that:
\begin{align*}
    \norm{T(M) - T(N)} \leq \frac{(1 + \norm{T(N)})\norm{M - N}}{\mu} \:.
\end{align*}
\end{prop}
\begin{proof}
Fix a unit norm $x$.
Define $f(u) = (1/2) x^\T M_{11} x + (1/2) u^\T M_{22} u + x^\T M_{12} u$
and $g(u) = (1/2) x^\T N_{11} x + (1/2) u^\T N_{22} u + x^\T N_{12} u$.
Let $u_\star = T(N) x$.
We have that
\begin{align*}
    \nabla f(u_\star) =  \nabla f(u_\star)  - \nabla g(u_\star) = (M_{22} - N_{22}) u_\star + (M_{12} - N_{12})^\T x \:.
\end{align*}
Hence, $\norm{\nabla f(u_\star)} \leq \norm{M_{12} - N_{12}} + \norm{M_{22} - N_{22}} \norm{u_\star}$.
We can bound $\norm{u_\star} = \norm{T(N) x} \leq \norm{T(N)}$.
The claim now follows using Lemma~\ref{lem:strongly_convex_min_perturb}.
\end{proof}

\begin{prop}
\label{prop:ordered_stability}
Let $K, K_0$ be two stabilizing policies for $(A, B)$.
Let $V, V_0$ denote their respective value functions and suppose that
$V \preceq V_0$.
We have that for all $k \geq 0$:
\begin{align*}
    \norm{(A+BK)^k} \leq \sqrt{\frac{\lambda_{\max}(V_0)}{\lambda_{\min}(S)}}  (1-\lambda_{\min}(V_0^{-1} S))^{k/2} \:.
\end{align*}
\end{prop}
\begin{proof}
This proof is inspired by the proof of Lemma 5.1 of \citet{abbasi18}.
Since $V$ is the value function for $K$, we have:
\begin{align*}
    V &= (A+BK)^\T V (A+BK) + S + K^\T R K \\
      &\succeq (A+BK)^\T V (A+BK) + S \:.
\end{align*}
Conjugating both sides by $V^{-1/2}$ and
defining $H := V^{1/2} (A+BK) V^{-1/2}$,
\begin{align*}
    I &\succeq V^{-1/2}(A+BK)^\T V (A+BK)V^{-1/2} + V^{-1/2} S V^{-1/2} \\
    &= H^\T H + V^{-1/2} S V^{-1/2} \:.
\end{align*}
This implies that $\norm{H}^2 = \norm{H^\T H} \leq \norm{I - V^{-1/2} S V^{-1/2}} = 1 - \lambda_{\min}(S^{1/2} V^{-1} S^{1/2}) \leq 1 - \lambda_{\min}(S^{1/2} V_0^{-1} S^{1/2})$.
The last inequality holds since $V \preceq V_0$ iff $V^{-1} \succeq V_0^{-1}$,
Now observe:
\begin{align*}
    \norm{ V^{1/2} (A+BK)^k V^{-1/2} } = \norm{H^k} \leq \norm{H}^k \leq (1-\lambda_{\min}(V_0^{-1} S))^{k/2}
\end{align*}
Next, for $M$ positive definite and $N$ square, observe that:
\begin{align*}
    \norm{M N M^{-1}} &= \sqrt{\lambda_{\max}(MNM^{-2} N^\T M)} \\
    &\geq \sqrt{\lambda_{\min}(M^{-2}) \lambda_{\max}( MNN^\T M) } \\
    &= \sqrt{\lambda_{\min}(M^{-2}) \lambda_{\max}(N^\T M^2 N) } \\
    &\geq \sqrt{\lambda_{\min}(M^{-2}) \lambda_{\min}(M^2) \norm{N}^2 } \\
    &= \frac{\norm{N}}{\kappa(M)} \:.
\end{align*}
Therefore, we have shown that:
\begin{align*}
    \norm{(A+BK)^k} \leq \sqrt{\kappa(V)} (1-\lambda_{\min}(V_0^{-1} S))^{k/2} \leq \sqrt{\frac{\lambda_{\max}(V_0)}{\lambda_{\min}(S)}}  (1-\lambda_{\min}(V_0^{-1} S))^{k/2} \:.
\end{align*}
\end{proof}

\begin{prop}
\label{prop:dlyap_norm_bound}
Let $A$ be a $(\tau, \rho)$ stable matrix, and let $\vertiii{\cdot}$ be either the 
operator or Frobenius norm.
We have that:
\begin{align}
    \vertiii{\dlyap(A, M)} \leq \frac{\tau^2}{1-\rho^2} \vertiii{M} \:.
\end{align}
\end{prop}
\begin{proof}
It is a well known fact that we can write $P = \sum_{k=0}^{\infty} (A^k)^\T M (A^k)$.
Therefore the bound follows from triangle inequality and the $(\tau, \rho)$ stability assumption.
\end{proof}

\begin{prop}
\label{prop:dlyap_perturbation}
Suppose that $A_1, A_2$ are stable
matrices. Suppose furthermore that $\norm{A_i^k} \leq \tau \rho^k$ for some $\tau \geq 1$
and $\rho \in (0, 1)$.
Let $Q_1, Q_2$ be PSD matrices.
Put $P_i = \dlyap(A_i, Q_i)$.
We have that:
\begin{align*}
    \norm{P_1 - P_2} \leq \frac{\tau^2}{1-\rho^2} \norm{Q_1-Q_2} + \frac{\tau^4}{(1-\rho^2)^2} \norm{A_1-A_2}(\norm{A_1}+\norm{A_2}) \norm{Q_2} \:.
\end{align*}
\end{prop}
\begin{proof}
Let the linear operators $F_1, F_2$ be such that
$P_i = F_i^{-1}(Q_i)$, i.e. $F_i(X) = X - A_i^\T X A_i$.
Then:
\begin{align*}
    P_1 - P_2 &= F_1^{-1}(Q_1) - F_2^{-1}(Q_2) \\
    &= F_1^{-1}(Q_1 - Q_2) + F_1^{-1}(Q_2) - F_2^{-1}(Q_2) \\
    &= F_1^{-1}(Q_1-Q_2) + (F_1^{-1} - F_2^{-1})(Q_2) \:.
\end{align*}
Hence $\norm{P_1-P_2} \leq \norm{F_1^{-1}}\norm{Q_1-Q_2} + \norm{F_1^{-1}-F_2^{-1}}\norm{Q_2}$.
Now for any $M$ satisfying $\norm{M} \leq 1$
\begin{align*}
    \norm{F_i^{-1}(M)} = \bignorm{ \sum_{k=0}^{\infty} (A_i^\T)^k M A_i^k } \leq \frac{\tau^2}{1-\rho^2} \:.
\end{align*}
Next, we have that:
\begin{align*}
    \norm{F_1^{-1} - F_2^{-1}} = \norm{F_1^{-1}(F_2-F_1)F_2^{-1}} \leq \norm{F_1^{-1}}\norm{F_2^{-1}}\norm{F_1-F_2} \leq \frac{\tau^4}{(1-\rho^2)^2} \norm{F_1-F_2} \:.
\end{align*}
Now for any $M$ satisfying $\norm{M} \leq 1$,
\begin{align*}
    \norm{F_1(M) - F_2(M)} &= \norm{A_2^\T M A_2 - A_1^\T M A_1} \\
    &= \norm{(A_2 - A_1)^\T M A_2 + A_1^\T M (A_2 - A_1)} \\
    &\leq \norm{A_1-A_2}( \norm{A_1}+\norm{A_2} ) \:.
\end{align*}
The claim now follows.
\end{proof}

\section{Useful Implicit Inversion Results}

\begin{prop}
\label{prop:recover_eps}
Let $T \geq 2$ and suppose that $\alpha \geq 1$.
Define $\varepsilon$ as:
\begin{align*}
	\varepsilon = \inf\left\{ \varepsilon \in (0, 1) : T \geq \frac{1}{\varepsilon^2} \log^\alpha(1/\varepsilon) \right\} \:,
\end{align*}
then we have 
\begin{align*}
	\varepsilon \leq \frac{\log^{(\alpha+1)/2}(T)}{\sqrt{T}} \:.
\end{align*}
As a corollary,
if $T \geq 2C$ then if we define $\varepsilon$ as:
\begin{align*}
	\varepsilon = \inf\left\{ \varepsilon \in (0, 1) : T \geq \frac{C}{\varepsilon^2} \log^\alpha(1/\varepsilon) \right\} \:,
\end{align*}
then we have
\begin{align*}
	\varepsilon \leq \sqrt{\frac{C}{T}} \log^{(\alpha+1)/2}(T/C) \:.
\end{align*}
\end{prop}
\begin{proof}
First, we know that such a $\varepsilon$ 
exists by continuity because $\lim_{\varepsilon \rightarrow 1^{-}}\frac{1}{\varepsilon^2} \log^\alpha(1/\varepsilon) = 0 $.

Suppose towards a contradiction that $\varepsilon > \log^{\beta}(T)/\sqrt{T}$
where $2\beta = \alpha + 1$. Note that
we must have $\log^{\beta}(T)/\sqrt{T} < 1$, since if we did not,
we would have
\begin{align*}
	1 \geq \varepsilon > \log^{\beta}(T)/\sqrt{T} \geq 1 \:.
\end{align*}
Therefore, by the definition of $\varepsilon$,
\begin{align*}
	T < \frac{T}{\log^{2\beta}(T)} \log^{\alpha}(\sqrt{T}/\log^\beta(T)) \leq \frac{T}{\log^{2\beta}(T)} \log^{\alpha}(\sqrt{T}) \:.
\end{align*}
This implies that:
\begin{align*}
\log^{2\beta}(T) \leq \log^{\alpha}(\sqrt{T}) = \frac{1}{2^\alpha} \log^\alpha(T) \:.
\end{align*}
Using the fact that $2\beta = \alpha + 1$, this implies:
\begin{align*}
	\log(T) \leq 1/2^\alpha \Longrightarrow T \leq \exp(1/2^{\alpha}) \leq \exp(1/2) \:.
\end{align*}
But this contradicts the assumption that $T \geq 2$.

The corollary follows from a change of variables $T \gets T/C$.
\end{proof}

\begin{prop}
\label{prop:T_helper}
Let $C \geq 1$ and $\alpha \geq 1$. We have that:
\begin{align*}
	\sup_{i=0, 1, 2, ...}  \frac{1}{(2^i)^{1/\alpha}} \polylog(C 2^i) \leq \mathrm{poly}(\alpha) \polylog(C) \:.
\end{align*}
\end{prop}
\begin{proof}
Let $\beta \geq 1$.
We have that:
\begin{align*}
	\frac{1}{(2^i)^{1/\alpha}} \log^\beta(C 2^i) &= \frac{1}{(2^i)^{1/\alpha}} (\log(C) + \log(2^i))^\beta \\
	&\leq \frac{2^{\beta-1}}{(2^i)^{1/\alpha}} (\log^\beta(C) + \log^\beta(2^i)) \\
	&\leq 2^{\beta-1} \log^\beta(C) + 2^{\beta-1} \frac{i^\beta}{(2^i)^{1/\alpha}} \log^\beta(2) \:.
\end{align*}
Next, we look at:
\begin{align*}
	f(i) := \frac{i^\beta}{(2^i)^{1/\alpha}} \:.
\end{align*}
We have that:
\begin{align*}
	\frac{d}{di} \log_2{f(i)} =  \frac{\beta}{i \log{2}} - \frac{1}{\alpha} \:.
\end{align*}
Setting the derivative to zero we obtain that $i = \alpha\beta/\log{2}$.
Therefore:
\begin{align*}
	\sup_{i=0,1,2,...} f(i) \leq \beta \left( \frac{\alpha\beta}{\log{2}}  \right)^\beta \:.
\end{align*}
The claim now follows.
\end{proof}

\begin{prop}
\label{prop:eps_simplification}
Let $C > 0$. Then for any $\varepsilon \in (0, \min\{1/e, C^2\})$, we have the following inequality holds:
\begin{align*}
    \varepsilon \log(1/\varepsilon) \leq C \:.
\end{align*}
As a corollary, let $M > 0$, then for $\varepsilon \in (0, \min\{M/e, C^2/M\})$ we have that:
\begin{align*}
    \varepsilon \log(M/\varepsilon) \leq C \:.
\end{align*}
\end{prop}
\begin{proof}
Let $f(\varepsilon) := \varepsilon \log(1/\varepsilon)$.
We have that $\lim_{\varepsilon \rightarrow 0^+} f(\varepsilon) = 0$ and that
$f'(\varepsilon) = \log(1/\varepsilon) - 1$.
Hence $f$ is increasing on the interval $\varepsilon \in [0, 1/e]$,
and $f(1/e) = 1/e$.
Therefore, if $C \geq 1/e$ then $f(\varepsilon) \leq C$ for any $\varepsilon \in (0, 1/e)$.

Now suppose that $C < 1/e$.
One can verify that the function $g(x) := 1/x + 2 \log{x}$
satisfies $g(x) \geq 0$ for all $x > 0$.
Therefore:
\begin{align*}
    g(C) \geq 0 &\Longleftrightarrow 1/C + 2 \log{C} \geq 0 \\
    &\Longleftrightarrow 1/C \geq \log(1/C^2) \\
    &\Longleftrightarrow C \geq C^2 \log(1/C^2) \\
    &\Longleftrightarrow f(C^2) \leq C \:.
\end{align*}
Since $C < 1/e$ we have $C^2 \leq C$ and therefore
$f(\varepsilon) \leq f(C^2) \leq C$ for all $\varepsilon \in (0, C^2)$.
This proves the first part.

To see the second part, use the variable substitution
$\varepsilon \gets \varepsilon/M$, $C \gets C/M$.
\end{proof}

\begin{prop}
\label{prop:lambert_w}
Let $\beta \geq 1$ and $C \geq (e/\beta)^\beta$. 
Let $x$ denote the solution to:
\begin{align*}
	x = C \log^\beta(x) \:.
\end{align*}
We have that $x \leq e^{(\alpha-1)\beta} \beta^\beta \cdot C \log^\beta(\beta C^{1/\beta})$, where $\alpha = 2 - \log(e-1)$.
\end{prop}
\begin{proof}
Let $W(\cdot)$ denote the Lambert $W$ function.
It is simple to check that $x = \exp(-\beta W(-\frac{1}{\beta C^{1/\beta}}))$
satisfies $x = C \log^\beta(x)$.
From Theorem 3.2 of \citet{alzahrani18}, we have that for any $t > 0$:
\begin{align*}
	W(-e^{-t-1}) > - \log(t+1) - t - \alpha \:, \:\: \alpha = 2 - \log(e-1) \:.
\end{align*}
We now write:
\begin{align*}
	W\left(-\frac{1}{\beta C^{1/\beta}}\right) &= W\left(-e^{\log(\frac{1}{\beta C^{1/\beta}})}\right) \\
	&= W\left( -e^{ - \log(\beta C^{1/\beta})  }    \right) \\
	&= W( -e^{-t-1} ) \:, \:\: t = \log(\beta C^{1/\beta}) - 1 \\
	&> - \log(t+1) - t - \alpha \:.
\end{align*}
where the last inequality uses the result from Alzahrani and Salem 
and the assumption that $C \geq (e/\beta)^\beta$.
We now upper bound $x$:
\begin{align*}
	x &= \exp\left(-\beta W\left(-\frac{1}{\beta C^{1/\beta}}\right)\right) \\
	&\leq \exp( \beta \log(t+1) + \beta t + \alpha \beta ) \\
	&= \exp( \beta \log\log(\beta C^{1/\beta}) ) \exp( \beta \log(\beta C^{1/\beta})) \exp((\alpha-1)\beta) \\
	&= \exp((\alpha-1)\beta) \beta^\beta C \log^\beta(\beta C^{1/\beta}) \:.
\end{align*}
\end{proof}


\section{Experimental Evaluation Details}
\label{sec:app:experiments}

In this section we briefly describe the other algorithms we evaluate in Section~\ref{sec:experiments}, and
also describe how we tune the parameters of these algorithms
for the experiments.

Define the function $J(K;W)$ as:
\begin{subequations}
\begin{align}
	J(K;W) &:= \lim_{T \to \infty} \E\left[ \frac{1}{T} \sum_{t=1}^{T} x_t^\T S x_t + u_t^\T R u_t \right]  \\
	&~~~~~~~~\text{s.t.} \:\: x_{t+1} = A x_t + B u_t + w_t \:, \:\: u_t = K x_t \:, \:\: w_t \sim \calN(0, W) \:.
\end{align}
\label{eq:app:lqr}
\end{subequations}

Certainty equivalence (nominal) control uses data to estimate a model $(\Ah, \Bh) \approx (A, B)$ 
and then solve for the optimal controller to \eqref{eq:app:lqr} via the Riccati equations.
On the other hand,
both policy gradients and DFO are derivative-free random search algorithms on 
$J(K;W)$.
For policy gradients, one uses action-space perturbations to 
obtain an unbiased estimate of the gradient of $J(K; \sigma_w^2 I + \sigma_\eta^2 BB^\T)$.
For DFO, random finite differences are used to obtain an unbiased estimate
of the gradient of $J_{\sigma_\eta}(K) := \E_{\xi}[J(K+\sigma_\eta \xi; \sigma_w^2 I)]$,
where each entry of $\xi$ is drawn i.i.d.\ from $\calN(0, 1)$.
Below, we describe each method in more detail.

\paragraph{Certainty equivalence (nominal) control.}
The certainty equivalence (nominal) controller solves \eqref{eq:app:lqr}
by first constructing an estimate $(\Ah, \Bh) \approx (A, B)$ and then 
outputting the estimated controller $\Kh$ via:
\begin{align*}
	\Kh &= - (\Bh^\T \Ph \Bh + R)^{-1} \Bh^\T \Ph \Ah \:, \\
	\Ph &= \mathsf{dare}(\Ah, \Bh, S, R) \:.
\end{align*}
The estimates $(\Ah, \Bh)$ are constructed via least-squares.
In particular, $N$ trajectories each of length $T$
are collected $\{ x_t^{(i)} \}_{t=1, i=1}^{T, N}$ using the random input sequence
$u_t^{(i)} \sim \calN(0, \sigma_u^2 I)$, and $(\Ah, \Bh)$ are formed as the solution to:
\begin{align*}
	(\Ah, \Bh) = \arg\min_{(A,B)} \frac{1}{2} \sum_{i=1}^{N} \sum_{t=1}^{T-1} \norm{x_{t+1}^{(i)} - A x_t^{(i)} - B u_t^{(i)}}^2 \:. 
\end{align*}
For our experiments, we set $\sigma_u = 1$.

\paragraph{Policy gradients.}
The gradient estimator works as follows. A large horizon length $T$ is fixed.
A trajectory $\{x_t\}$ is rollout out for $T$ timesteps with the input sequence $u_t = K x_t + \eta_t$,
with $\eta_t \sim \calN(0, \sigma_\eta^2 I)$.
Let $\tau_{s:t} = (x_s, u_s, x_{s+1}, u_{s+1}, ..., x_{t}, u_{t})$ denote a
sub-trajectory, and let $c(\tau_{s:t})$ denote the LQR cost over this sub-trajectory,
i.e. $c(\tau_{s:t}) = \sum_{k=s}^{t} x_k^\T S x_k + u_k^\T R u_k$.
The policy gradient estimate is:
\begin{align*}
	\widehat{g} = \frac{1}{T} \sum_{t=1}^{T} \frac{c(\tau_{t:T})}{\sigma_\eta^2} \eta_t x_t^\T \:.
\end{align*}
Of course, one can use a baseline function $b(\tau_{1:{t-1}}, x_t)$ for variance reduction as follows:
\begin{align*}
	\widehat{g} = \frac{1}{T} \sum_{t=1}^{T} \frac{c(\tau_{t:T}) - b(\tau_{1:{t-1}}, x_t)}{\sigma_\eta^2} \eta_t x_t^\T \:.
\end{align*}

\paragraph{DFO.}
We use the two point estimator. As in policy gradients, we fix a horizon length $T$.
We first draw a random perturbation $\xi$.
Then, we rollout one trajectory $\{x_t\}_{t=1}^{T}$ with $u_t = (K + \sigma_\eta \xi) x_t$,
and we rollout another trajectory $\{x'_t\}_{t=1}^{T}$ with $u'_t = (K - \sigma_\eta \xi) x'_t$.
We then use the gradient estimator:
\begin{align*}
	\widehat{g} = \frac{\frac{1}{T}\sum_{t=1}^{T} c_t - \frac{1}{T}\sum_{t=1}^{T} c'_t}{2\sigma_\eta} \xi \:, \:\:
	c_t = x_t^\T S x_t + u_t^\T R u_t \:, \:\:
	c'_t = (x'_t)^\T S x'_t + (u'_t)^\T R u'_t \:. 
\end{align*}

\paragraph{MFLQ.}
We update the policy every $100$ iterations and do not execute a random exploratory action since we found that it negatively affected the performance of the algorithm in practice. In terms of the parameters described in Algorithm~1 of \citet{abbasi18} we execute v2 of the algorithm and set $T_s = \infty$ and $T_v = 100$. We also chose to use all data collected throughout an experiment when updating the policy.

\paragraph{Optimal.}
The optimal controller simply solves for the optimal controller to \ref{eq:app:lqr} given the true matrices $A$ and $B$. That is, it uses the controller
\begin{align*}
	K &= - (B^\T P B + R)^{-1} B^\T P A \:, \\
	P &= \mathsf{dare}(A, B, S, R) \:.
\end{align*}

\paragraph{Offline setup details.}
Recall that we use stochastic gradient descent with a constant step size $\mu$
as the optimizer for both policy gradients and DFO.
After every iteration, we project the iterate $K_t$ onto the set
$\{ K : \norm{K}_F \leq 5 \norm{K_\star}_F \}$, where $K_\star$ is
the optimal LQR controller (we assume the value $\norm{K_\star}_F$ is known for simplicity).
We tune the
parameters of each algorithm as follows.
We consider a grid of step sizes $\mu$ given by $[10^{-3}, 10^{-4}, 10^{-5}, 10^{-6}]$
and a grid of $\sigma_\eta$'s given by $[1, 10^{-1}, 10^{-2}, 10^{-3}]$.
We fix the rollout horizon length $T = 100$ and choose the pair of
$(\sigma_\eta, \mu)$ in the grid which yields the lowest cost after
$10^6$ timesteps. This resulted in the pair
$(\sigma_\eta, \mu) = (1, 10^{-5})$ for policy gradients
and 
$(\sigma_\eta, \mu) = (10^{-3}, 10^{-4})$ for DFO.
As mentioned above, we use the two point evaluation for derivative-free optimization,
so each iteration requires $2T$ timesteps.
For policy gradient, 
we evaluate two different baselines $b_t$.
One baseline, which we call the \emph{simple} baseline, 
uses the empirical average cost $b = \frac{1}{T} \sum_{t=1}^{T} c_t$
from the previous iteration as a constant baseline.
The second baseline, which we call the \emph{value function} baseline,
uses $b(x) = x^\T V(K) x$ with $V(K) = \dlyap(A + BK, S + K^\T R K)$ as the baseline.
We note that using this baseline requires exact knowledge of the dynamics $(A, B)$;
it can however be estimated from data at the expense of additional sample complexity
(c.f. Section~\ref{sec:results:lstdq}). For the purposes of this experiment, we 
simply assume the baseline is available to us.

\paragraph{Online setup details.}
In the online setting we warm-start every algorithm by first collecting $2000$ datapoints collected
by feeding the input $Kx_t + \eta_t$ to the system where $K$ is a stabilizing controller and $\eta_t$
is Gaussian distributed additive noise with standard deviation $1$. We then run each algorithm for $10,000$ iterations.
In the case of LSPI we set the initial number of policy iterations $N$ to be $3$ and subsequently increase it to $4$ at $2000$ iterations, $5$ at $4000$ iterations, and $6$ at $6000$ iterations. We also 
follow the experimental methodology of \citet{dean18} and 
set $T_i = 10(i + 1)$ and set $\sigma^2_{\eta, i} = 0.01\left(\frac{1}{i + 1}\right)^{2/3}$ where $i$ is the epoch number. Finally we repeat each experiment for $100$ trials.

\end{document}